\documentclass{article}
\usepackage[
	backend=biber,
	style=authoryear
]{biblatex}
\usepackage{amsthm}
\usepackage{DotArrow}
\usepackage{amsmath}
\usepackage{amssymb}
\usepackage{mathrsfs}
\usepackage{fix-cm}
\usepackage{graphicx}
\usepackage{algpseudocode}
\usepackage{algorithm}
\usepackage{authblk}

\theoremstyle{definition}
\newtheorem{theorem}{Theorem}
\newtheorem*{theorem*}{Theorem}
\newtheorem{definition}[theorem]{Definition}
\newtheorem*{definition*}{Definition}

\newtheorem*{propsition*}{Propsition}
\newtheorem{lemma}[theorem]{Lemma}
\newtheorem*{lemma*}{Lemma}
\newtheorem{corollary}{Corollary}
\newtheorem*{corollary*}{Corollary}

\newcommand{\indep}{\mathrel{\perp\mspace{-10mu}\perp}}

\begin{document}
	
\title{\textbf{Estimating the average causal effect of intervention in continuous variables using machine learning
}}

\author[ ]{Yoshiaki Kitazawa}
\affil[ ]{NTT DATA Mathematical Systems Inc.\\
		 Data Mining Division\\
	1F Shinanomachi Rengakan, 35, Shinanomachi,\\
	 Shinjuku-ku, Tokyo, 160-0016, Japan\\	
}
\affil[ ]{kitazawa@msi.co.jp}
\date{}

\maketitle

\begin{abstract}
The most widely discussed methods for estimating the Average Causal Effect/Average Treatment
Effect are those for intervention 
in discrete binary variables whose value represents intervention/non-intervention groups.
On the other hand, methods for intervening in continuous variables independent 
of data generating models have not been developed.
In this study, we give a method for estimating the average causal effect for 
intervention in continuous variables that can be applied to data of any generating models as 
long as the causal effect is identifiable.
The proposing method is independent of machine learning algorithms and preserves the identifiability of data.
\end{abstract}

\section{Introduction}

The causal effect is defined by Pearl's $do$ operation as a probability distribution over observed
data in the way that it is altered from one which generates data originally  [\cite{pearl1995causal,pearl2009causality}]. 
When dealing with causal effects in real-world problems,  it is also necessary to take into account unobserved variables
that is not included in data. 
In general, causal effects are counterfactual probability distributions that differ from data generating systems in the real world.
When we consider the existence of unobserved data, it becomes a problem if it can be determined by observed data available.
That is, we need to consider the identifiability of causal effects in this case.
This problem has recently been resolved to a certain extent [\cite{tian2002general,shpitser2006identification,shpitser2012identification}].

If causal effects are identifiable, the next problems are how to estimate the (conditional) average causal effect, 
that is, how to calculate the (conditional) expected value of causal effects.
Recently some methods have been proposed for this estimation, 
such as one by reweighting probability [\cite{jung2020estimating}] or one by assuming
a semiparametric model [\cite{chernozhukov2018double}].
However, Any methods that is independent of models generating 
data and available for intervention in continuous variables 
have not been developed yet.

In addition, related to estimation of the average causal effects, it has been widely discussed 
in recent years to estimate of (Conditional) Average Treatment Effects that are the difference 
between the average causal effects of intervention / non-intervention groups.
The major methods of these, for example,
are the re-weighting mothod by propensity score such as Inverse propensity weighting (IPW) 
method [\cite{chernozhukov2018double,rosenbaum1984reducing}], 
the matching method [\cite{rosenbaum1983central,rosenbaum1985constructing}], 
the stratification method 
[\cite{rosenbaum1984reducing,imbens2004nonparametric}], 
the tree-based method such that Bayesian Additive Regression Trees (BART)
[\cite{chipman2006bayesian,chipman2010bart,hill2011bayesian}].
In addition, methods of meta-machine learning algorithm have been proposed in recent years
such that T-Learner [\cite{kunzel2019metalearners}], 
S-Learner [\cite{kunzel2019metalearners}], 
X-Learner [\cite{kunzel2019metalearners}], 
U-Learner [\cite{nie2021quasi}] 
and R-Learner [ \cite{nie2021quasi}]. 
Currently, when estimating the average treatment effects of intervention in continuous variables,
for example patient responses to drug dose, the above methods for discrete binary intervention
are only used,  by extending it to available for multi-values intervention [\cite{schwab2020learning}].
However, when estimating the average causal effects of continuous value intervention, it is 
necessary that estimators of them have to be continuous functions for values of intervention,
whereas the estimators of above methods for discrete intervention is not.

In this paper, we propose a method that enables estimation of the average causal effect
of intervention in continuous variables.
The proposing method use supervised machine learning algorithms but is independent of  
algorithms to use.
In addition, it preserves the identifiability of causal effects for models generating
data, and is designed to make as few assumptions about training data as possible.

This article is divided to five parts. First, we introduce the background of study.
Second, we define terminologies and concepts necessary for causal inference.
Third, we state the Main Theorem and present the proposing method for estimating 
the average causal effects of intervention in continuous variables.
Fourth, we report results of a numerical experiment for the proposing methods and discuss them.
%Fifth, We discuss the proposing methods.
Lastly, we conclude this paper.

\section{Notation and Definitions}\label{sec_term_def}

We denote random variables by capital letters, $A$.
Small letters, $a$, represents a values of random variables corresponding, $A$.
Bold letters, $\mathbf{A}$ or $\mathbf{a}$, represent a set of variables or values of variables.
In particular, we use $\mathbf{V}=\{V_1,V_2,\dots,V_n\}$ for observed variables and
$\mathbf{U}=\{U_1, U_2, \dots, U_m\}$ for unobserved variables.
For sake of simplicity,  we assume that any $V \in \mathbf{V}$ is not determined by the others,
that is, $V \neq f(\mathbf{W})$ for any function $f$ and any $\mathbf{W}\subset \mathbf{V} \setminus V$.
We denote the domain of a variable $A$ by $\mathcal{X}_A$.
For sake of simplicity, we assume that $P(v) > 0$ for all $V \in \mathbf{V}$ and $v\in\mathcal{X}_V$.
For a set of variables $\mathbf{A}$, 
let $\mathcal{X}_\mathbf{A} = \mathcal{X}_{A_1} \times \cdots \times \mathcal{X}_{A_n}$. 
In this paper, $\mathbf{V}\cup\mathbf{U}$ is a semi-Markovian model, 
and a Directed Acyclic Graph(DAG) $G=G_{\mathbf{V}\mathbf{U}}$ is a causal graph for them.
Each $Pa(\mathbf{A})_G$, $Ch(\mathbf{A})_G$, $An(\mathbf{A})_G$ and $De(\mathbf{A})_G$  
represents the parents, children, ancestors and descendants of observed variables in $G$
for $\mathbf{A}\subset\mathbf{V}$.
$UPa(\mathbf{A})_G$ represents the parents of unobserved variables in $G$.
In this paper, $Pa(\mathbf{A})_G$, $Ch(\mathbf{A})_G$, 
$An(\mathbf{A})_G$, $De(\mathbf{A})_G$ doesn't 
include $\mathbf{A}$ itself.
The set of observed variables which has no paths to each $\mathbf{A}$ in $G$
is denoted by $Ind(\mathbf{A})_G :=\{V\in\mathbf{V}|(V\indep \mathbf{A})_G \}$.
$\overline{G}(\mathbf{X})$ is denoted the graph obtained from $G$ by deleting
all arrows emerging from variables to $\mathbf{X}$. 
$\underline{G}(\mathbf{X})$ is denoted the graph obtained from $G$ by 
deleting all arrows from $\mathbf{X}$ to variables. 
$\underline{\overline{G}}(\mathbf{X}_1, \mathbf{X}_2)$ is denoted the graph obtained from $G$ 
by deleting all arrows emerging from variables to $\mathbf{X}_1$ 
and all arrows emerging from $\mathbf{X}_2$ to variables. 
$G^{obs}$ is denoted the graph obtained from $G$
by remaining only arrows between observed variables.

According to $G$, a probability distribution $P$ of $\mathbf{V}\cup\mathbf{U}$
can be decomposed into as below. 
\begin{eqnarray}
	\lefteqn{P(V_1=v_1,\dots,V_n=v_n, U_1=u_1,\dots,U_m=u_m)} \quad \nonumber \\
	&=& \prod_{i} P(V_i=v_i|Pa(V_i)_G=pa_i, UP(V_i)=upa_i), \nonumber %\label{def_causaldiarog_factorize_eq_with_unobserv_vals}
\end{eqnarray}
where $pa_i \in \mathcal{X}_{Pa(V_i)_G}$ is values of $Pa(V_i)_G$,  
$upa_i \in \mathcal{X}_{UPa(V_i)_G}$ is values of  $UPa(V_i)_G$. %in left hand side of $v_1,\dots,v_n$. 

The probability distribution for only observed variables
is obtained by marginalized this distribution over all unobserved variables. 
\begin{eqnarray}
	\lefteqn{P(V_1=v_1,\dots,V_n=v_n)}\quad \nonumber \\ 
	&=& \sum_{\mathbf{u}\in\mathcal{X}_\mathbf{U}} \prod_{i} P(V_i=v_i|Pa(V_i)_G=pa_i, UPa(V_i)_G=upa_i)  \nonumber \\
	&=& \prod_{i} P(V_i=v_i|Pa(V_i)_G=pa_i). \nonumber
	%\label{def_causaldiarog_factorize_eq_only_observe_vals}
\end{eqnarray}

Give two disjoint sets of $\mathbf{X}, \mathbf{Y} \subset \mathbf{V}$,
the causal effect of  $\mathbf{X}$ on $\mathbf{Y}$, denoted by 
$P(\mathbf{Y}=\mathbf{y}|do(\mathbf{X}=\mathbf{x}))$, is defined 
as the probability distribution as follows. 
\begin{equation}
	P(\mathbf{Y}=\mathbf{y}|do(\mathbf{X}=\mathbf{x})) =
	\sum_{\mathbf{v}' \in \mathcal{X}_{\mathbf{V}'}}
	\frac{
		P(\mathbf{Y}=\mathbf{y}, \mathbf{X}=\mathbf{x}, \mathbf{V}'=\mathbf{v}')}{
		P(\mathbf{X}=\mathbf{x}|Pa(\mathbf{X})_G=pa_\mathbf{x})
	}. \label{def_causal_effect_no_condi_eq}
\end{equation}
where, $\mathbf{V}'=\mathbf{V}\setminus(\mathbf{X}\cup\mathbf{Y})$ and
$pa_\mathbf{x}$ represents values of $Pa(\mathbf{X})_G$.

Give disjoint sets of $\mathbf{X}, \mathbf{Y}, \mathbf{Z} \subset \mathbf{V}$,
the causal effect of $\mathbf{X}$ on $\mathbf{Y}$ under conditions $\mathbf{Z}$,
denoted by $P(\mathbf{Y}=\mathbf{y}|do(\mathbf{X}=\mathbf{x}),\mathbf{Z}=\mathbf{z})$, is defined 
as the probability distribution as follows. 
\begin{equation}
	P(\mathbf{Y}=\mathbf{y}|do(\mathbf{X}=\mathbf{x}),\mathbf{Z}=\mathbf{z}) =
	\frac{
		P(\mathbf{Y}=\mathbf{y},\mathbf{Z}=\mathbf{z}|do(\mathbf{X}=\mathbf{x}))}{
		P(\mathbf{Z}=\mathbf{z}|do(\mathbf{X}=\mathbf{x}))
	} \label{def_causal_effect_on_condi_eq}
\end{equation}

\section{Main Theorem}\label{sec_main_theorem}
In this paper, We  consider the conditional expectations of causal effects 
$P(\mathbf{Y}=\mathbf{y}|do(\mathbf{X}=\mathbf{x}),\mathbf{Z}=\mathbf{z})$.
That is, We consider the probability distribution $P'$ following
%For disjoint sets of $\mathbf{X}, \mathbf{Y}, \mathbf{Z} \subset \mathbf{V}$,
%let $P'$  be
\begin{equation}
	P'=P(\mathbf{Y}|do(\mathbf{X}),\mathbf{Z})\cdot P(X_1)\cdot P(X_2) \cdot \cdots \cdot P(X_n)\cdot P(\mathbf{Z}), \label{eq_dist_ave_causal}
\end{equation}
and we consider the expectation 
\begin{equation}
	E_{P'}[\mathbf{Y}=\mathbf{y}|\mathbf{X}=\mathbf{x}, \mathbf{Z}=\mathbf{z}]. \nonumber\\ %\label{kainyukouka-condi-exp}
\end{equation}
Here, $\mathbf{X}$ are variables to intervene in, $\mathbf{Z}$ are covariates
and $\mathbf{Y}$ are target variables to estimate.

\subsection{Graph structure that can be used for modeling the average casual effects}

We give a definition of a graph structure that that can be used for modeling the average casual
effects, that is, the conditional expectations for the probability distribution after intervention.
\begin{definition}[Availability for modeling the average casual effects]
	 For disjoint sets $\mathbf{Y},\mathbf{X}\subset \mathbf{V}$,
	 a DAG $G$ is said to be available for modeling the average casual effects 
	 $P(\mathbf{Y}|do(\mathbf{X}),\mathbf{Z})$, 
	 if  $G$ satisfies following two conditions	
	\begin{enumerate}
		\item $P(\mathbf{Y}|do(\mathbf{X}),\mathbf{Z})$ is identifiable in $G$.
		\item $\overline{G^{obs}}(\mathbf{X})=G^{obs}$.
	\end{enumerate}	
\end{definition}

If a graph is available for modeling the average casual effects $E[\mathbf{Y}=\mathbf{y}|do(\mathbf{X}=\mathbf{x}]$,
then $P'$ of (\ref{eq_dist_ave_causal}) is as follows.
\begin{eqnarray}
	P'  &=& P(\mathbf{Y}=\mathbf{y}|do(\mathbf{X}=\mathbf{x}), \mathbf{Z}=	\mathbf{z})  \nonumber \\
		&& \quad \times P(X_1=x_1) \cdot P(X_2=x_2) \cdot \cdots \cdot P(X_n=x_n) \cdot P(\mathbf{Z}=\mathbf{z})\nonumber \\
		&=& P(\mathbf{Y}=\mathbf{y}, \mathbf{X}=\mathbf{x}, \mathbf{Z}=\mathbf{z}). \nonumber
\end{eqnarray}
That is, there is no difference of joint probability distributions
over $\mathbf{Y},\mathbf{X}$ and $\mathbf{Z}$ between before and after intervention in $\mathbf{X}$.
Thus, if a model generating a data has this structure, we can use the original data for modeling the
average causal effects.

Figure \ref{fig_able_to_build_modeling} is an example
of a graph available for modeling the average casual effects 
$E[Y|do(X), Z_1, Z_2]$.
Figure \ref{fig_not_aible_to_build_modeing} is an example of one not available for that.
Comparing two graphs, the graph not available for modeling has the arrows (red) between
$X$ and $Z_1, Z_2$ that hinder modeling the average causal effect $E[Y|do(X), Z_1, Z_2]$, 
but one available for modeling has been deleted them.
If a data is generated by models of Figure \ref{fig_able_to_build_modeling},
it is available for modeling the average causal effect $E[Y|do(X), Z_1, Z_2]$.
On the other hand, if a data is not generated by models of Figure \ref{fig_able_to_build_modeling},
for example one generated by models of Figure \ref{fig_not_aible_to_build_modeing}, it is not available for the modeling.
\begin{figure}
	\begin{center}
		\includegraphics[keepaspectratio, scale=0.3]{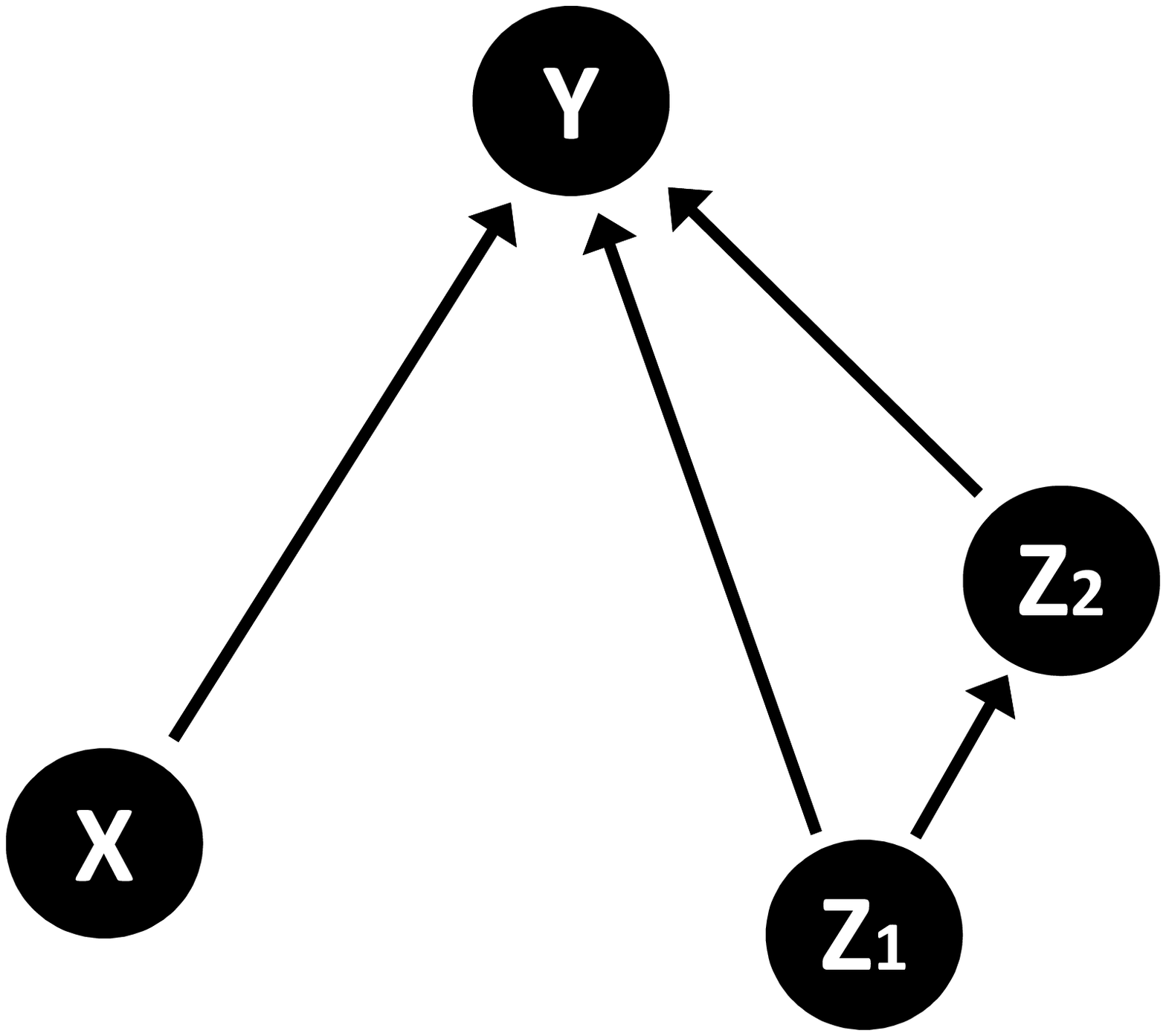}
		\caption{A graph available for modeling the average casual effects. If a data is generated by models of
			this, it is available for modeling the average causal effect
			$E[Y|do(X), Z_1, Z_2]$ without any operations.}\label{fig_able_to_build_modeling}
		\includegraphics[keepaspectratio, scale=0.3]{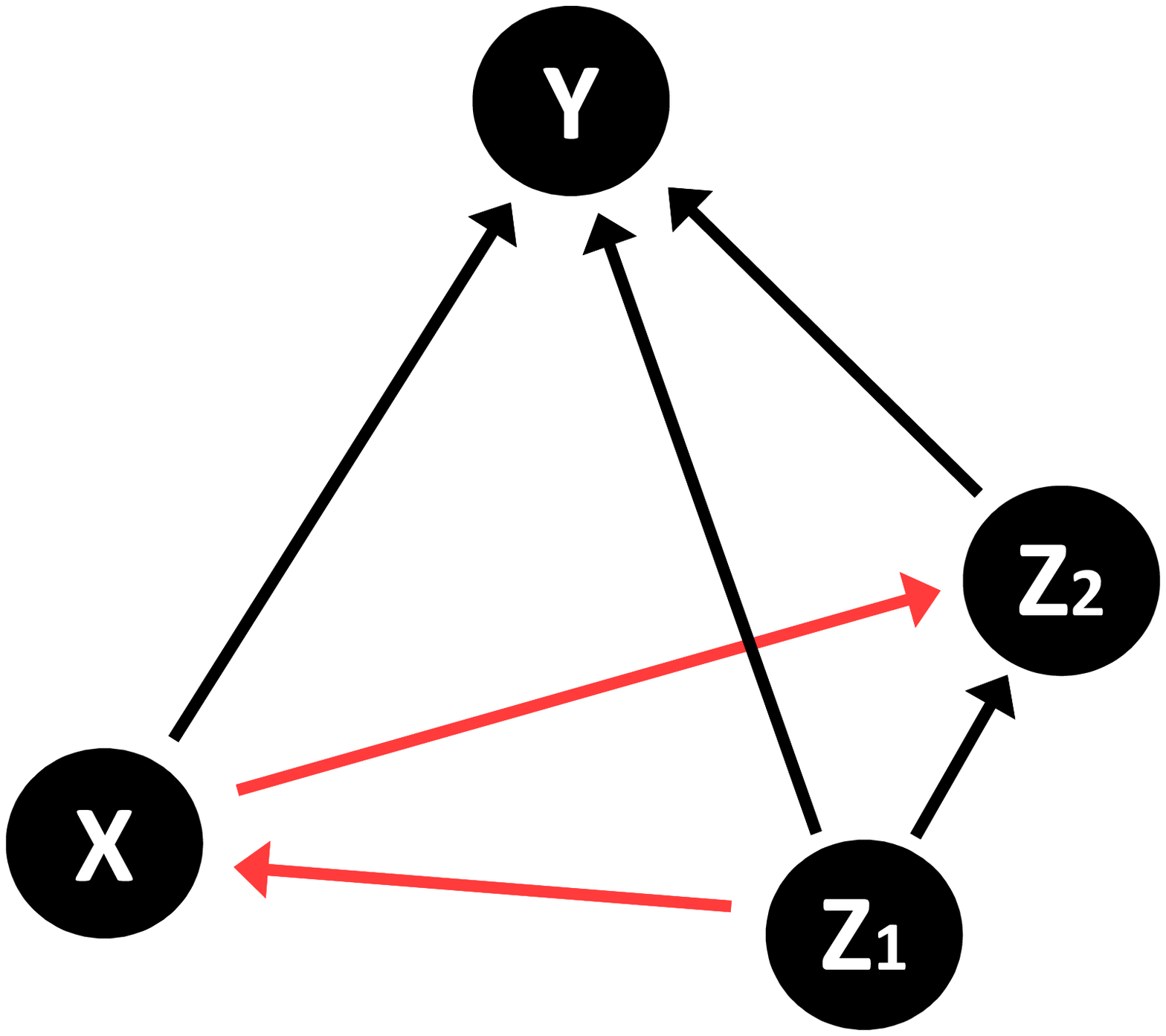}
		\caption{A graph not available for modeling the average causal effect.
			If a data is not generated by models of this, it is not available for modeling the average causal effect
			$E[Y|do(X), Z_1, Z_2]$.}\label{fig_not_aible_to_build_modeing}
	\end{center}
\end{figure}

\subsection{Main theorem}

The following Theorem \ref{thorem_multi_vals} shows how the proposing method builds
models from original data.
In addition, it shows that the method preserves for the
identifiability of causal effects for models which generate data.
This theorem \ref{thorem_multi_vals} includes cases with more than 
two variables to intervene in.
That is, it includes the causal effects of joint intervention.
\begin{theorem}\label{thorem_multi_vals}
	For $Y \in \mathbf{V}$ and $\mathbf{X}=\{X_1, X_2, \dots, X_n\} \subset \mathbf{V}$, 
	let $\mathbf{Z}=\mathbf{V}\setminus (\{Y\}\cup\mathbf{X})$.
	Suppose that $\mathbf{X} \subset An(Y)_G$ and that $\mathbf{Z} \cap De(\mathbf{X})_G \cap De(Y)_G=\phi$.
	Then, for $\mathbf{X}$, $x_1\in \mathcal{X}_{X_1}$, $x_2\in \mathcal{X}_{X_2}$, $\dots$, $x_n\in \mathcal{X}_{X_n}$
	and $\mathbf{z}\in\mathcal{X}_{\mathbf{Z}}$, let 
	$\widetilde{\mathbf{X}}^{(k)}=\{X_1^{(k)}, X_2^{(k)}, \dots, X_n^{(k)}\}\ (1 \le k \le n)$ be as follows. \\
	\noindent \underline{If $k = 1$}
	\begin{eqnarray}
		\widetilde{X}_i^{(1)} &=& X_i - E[X_i|\mathbf{Z}=\mathbf{z}]  \quad \text{for $1 \le i \le n$} \nonumber\\
		\widetilde{x}_i^{(1)} &=& x_i - E[X_i|\mathbf{Z}=\mathbf{z}] \; \quad \text{for $1 \le i \le n$}. \nonumber
	\end{eqnarray}
	\noindent \underline{If $2 \le k \le n-1$}
	\begin{eqnarray}
		\widetilde{X}_i^{(k+1)} &=&
		\begin{cases} 
			\widetilde{X}_i ^{(k)}  & \text{if  $1 \le i \le k$} \nonumber\\
			\widetilde{X}_{i}^{(k)} - E[\widetilde{X}_{i}^{(k)}|\widetilde{X}_{k}^{(k)}=\widetilde{x}_{k}^{(k)}] & \text{if $i \ge k+1$} \nonumber		
		\end{cases} \nonumber\\
		\widetilde{x}_i^{(k+1)} &=&
		\begin{cases} 
			\widetilde{x}_i ^{(k)} & \text{if  $1 \le i \le k$} \nonumber\\
			\widetilde{x}_{i}^{(k)} - E[\widetilde{X}_{i}^{(k)}|\widetilde{X}_{k}^{(k)}=\widetilde{x}_{k}^{(k)}] & \text{if $i \ge k+1$}. \nonumber			
		\end{cases}	
	\end{eqnarray}
	Using the above results, for $Y$ and $y\in \mathcal{X}_Y$, let $\widetilde{Y}^{(k)}\ (1 \le k \le n)$ be as follows. \\
	\noindent \underline{If $k = 1$}
	\begin{eqnarray}
		\widetilde{Y}^{(1)} &=& Y - E[Y|\mathbf{Z}=\mathbf{z}] \nonumber\\
		\widetilde{y}^{(1)} &=& y - E[Y|\mathbf{Z}=\mathbf{z}]. \nonumber
	\end{eqnarray}	
	\noindent \underline{If $2 \le k \le n-1$}
	\begin{eqnarray}
		\widetilde{Y}^{(k+1)} &=& \widetilde{Y}^{(k)} - E[\widetilde{Y}^{(k)}|\widetilde{X}_k^{(k)}=\widetilde{x}_k^{(k)}] \nonumber\\
		\widetilde{y}^{(k+1)} &=& \widetilde{y}^{(k)} - E[\widetilde{Y}^{(k)}|\widetilde{X}_k^{(k)}=\widetilde{x}_k^{(k)}]. \nonumber
	\end{eqnarray}
	Additionally, for each $U \in \mathbf{U}$, let $\widetilde{U}^{(k)}\ (1 \le k \le n)$ be as follows. \\
	\begin{eqnarray}
		\widetilde{U}^{(1)} &=& U - E[U|\mathbf{Z}=\mathbf{z}] \nonumber \\
		\widetilde{U}^{(2)} &=& \widetilde{U}^{(1)} - E[\widetilde{U}^{(1)}|\widetilde{X}_1^{(1)}=\widetilde{x}_1^{(1)}] \nonumber \\
		&\cdots& \nonumber \\
		\widetilde{U}^{(n)} &=& \widetilde{U}^{(n-1)} - E[\widetilde{U}^{(n-1)}|\widetilde{X}_{n-1}^{(n-1)}=\widetilde{x}_{n-1}^{(n-1)}] \nonumber
	\end{eqnarray}
	Let $\widetilde{Y}=\widetilde{Y}^{(n)}$, 
	$\widetilde{\mathbf{X}}=\widetilde{\mathbf{X}}^{(n)} = \{\widetilde{X}_1^{(1)},\widetilde{X}_2^{(2)},\dots,\widetilde{X}_n^{(n)}\}$, 
	$\widetilde{\mathbf{U}}=\{\widetilde{U}^{(n)}|U\in\mathbf{U}\}$, and let
	$\widetilde{\mathbf{V}}=\{\widetilde{Y}\}\cup\widetilde{\mathbf{X}}\cup\mathbf{Z}$.
	Then, $\widetilde{\mathbf{V}}\cup\widetilde{\mathbf{U}}$ is a semi-Markovian model.
	Let $\widetilde{G}$ be a DAG for $\widetilde{\mathbf{V}}\cup\widetilde{\mathbf{U}}$, 
	then $\widetilde{\mathbf{X}} \subset An(\widetilde{Y})_{\widetilde{G}} \cup Ind(\widetilde{Y})_{\widetilde{G}}$.
	If $P(Y|do(\mathbf{X}),\mathbf{Z})$ is identifiable in $G$, then
	$\widetilde{G}$ is available for modeling the average causal effect $E[\widetilde{Y}|do(\widetilde{\mathbf{X}})]$.
	
	Moreover, it holds that
	\begin{eqnarray}
		\lefteqn{P(Y=y|do(X_1=x_1),\dots,do(X_n=x_n), \mathbf{Z}=\mathbf{z})} \quad \nonumber \\
		&=& P(\widetilde{Y}^{(n)}=\widetilde{y}^{(n)}|
		do(\widetilde{X}_{1}^{(1)}=\widetilde{x}_{1}^{(1)}), do(\widetilde{X}_{2}^{(2)}=\widetilde{x}_{2}^{(2)}), \nonumber \\
		&& \qquad \mbox \dots,	do(\widetilde{X}_{n}^{(n)}=\widetilde{x}_{n}^{(n)})) \label{eq_tilde_is_same_original} \\
		&=& P(\widetilde{Y}^{(n)}=\widetilde{y}^{(n)}|\widetilde{X}_{1}^{(1)}=\widetilde{x}_{1}^{(1)}, \widetilde{X}_{2}^{(2)}=\widetilde{x}_{2}^{(2)}, \nonumber \\
		&& \qquad \dots,\widetilde{X}_{n}^{(n)}=\widetilde{x}_{n}^{(n)}) \label{eq_do_equals_codi} \\
		&=& P(\widetilde{Y}^{(n)}=\widetilde{y}^{(n)}|\widetilde{X}_{n}^{(n)}=\widetilde{x}_{n}^{(n)}). \label{for_collolary_trans}
	\end{eqnarray}		
\end{theorem}

The following corollary shows how to estimate the average causal effects from each of
the conditional expectations (predictions of supervised algorithms) obtained
from Theorem \ref{thorem_multi_vals}.
\begin{corollary}\label{corollary_how_to_calc_ave_intervention_effect}
	For 
	$\widetilde{Y}^{(1)}, \widetilde{Y}^{(2)}, \ldots \widetilde{Y}^{(n)}$, 
	$\widetilde{\mathbf{y}}^{(1)}, \widetilde{\mathbf{y}}^{(2)}, \ldots, \widetilde{\mathbf{y}}^{(n)}$,
	$\widetilde{\mathbf{X}}^{(1)}, \widetilde{\mathbf{X}}^{(2)}, \ldots, \widetilde{\mathbf{X}}^{(n)}$ 
	and $\widetilde{\mathbf{x}}^{(1)}, \widetilde{\mathbf{x}}^{(2)}, \ldots, \widetilde{\mathbf{x}}^{(n)}$
	in Theorem \ref{thorem_multi_vals} 
	and for $Y$, $y$, $X_1, X_2, ..., X_n$, $x_1, x_2, \ldots, x_n$, $\mathbf{Z}$ and $\mathbf{z}$, 
	let
	\begin{eqnarray}
		\lefteqn{P'(Y=y, X_1=x_1, X_2=x_2,\ldots X_n=x_n, \mathbf{Z}=\mathbf{z}) } \quad \nonumber \\
		&=&P(Y=y|do(X_1=x_1), do(X_2=x_2), \nonumber \\
		&& \quad \ldots, do(X_n=x_n), \mathbf{Z}=\mathbf{z}) \nonumber \\
		&& \times P(X_1=x_1) \cdot P(X_2=x_2) \cdots P(X_n=x_n)  \nonumber \\
		&& \times P(\mathbf{Z} = \mathbf{z}), \nonumber
	\end{eqnarray}
	and let
	\begin{eqnarray}
		\lefteqn{P''(Y=y, \widetilde{X}_{1}^{(1)}=\widetilde{x}_{1}^{(1)}, \widetilde{X}_{2}^{(2)}=\widetilde{x}_{2}^{(2)}, \dots,\widetilde{X}_{n}^{(n)}=\widetilde{x}_{n}^{(n)}, \mathbf{Z}=\mathbf{z})} \quad \nonumber \\
		&=& P(Y=y|\widetilde{X}_{1}^{(1)}=\widetilde{x}_{1}^{(1)}, \widetilde{X}_{2}^{(2)}=\widetilde{x}_{2}^{(2)}, \dots,\widetilde{X}_{n}^{(n)}=\widetilde{x}_{n}^{(n)}) \nonumber \\
		&& \times  P(\widetilde{X}_{1}^{(1)}=\widetilde{x}_{1}^{(1)})\cdot P(\widetilde{X}_{2}^{(2)}=\widetilde{x}_{2}^{(2)})
		\cdots P(\widetilde{X}_{n}^{(n)}=\widetilde{x}_{n}^{(n)}) \nonumber \\
		&& \times P(\mathbf{Z} = \mathbf{z}). \nonumber
	\end{eqnarray}
	Then,
	\begin{eqnarray}
		\lefteqn{E_{P'}[Y|X_1=x_1, X_2=x_2,\ldots,X_n=x_n,\mathbf{Z}=\mathbf{z}]}  \nonumber\\
		&=& E_{P''}[Y|\widetilde{X}_{1}^{(1)}=\widetilde{x}_{1}^{(1)}, \widetilde{X}_{2}^{(2)}=\widetilde{x}_{2}^{(2)}, \nonumber \\
		&& \qquad \ldots,\widetilde{X}_{n}^{(n)}=\widetilde{x}_{n}^{(n)},\mathbf{Z}=\mathbf{z}] \label{eq_condi_ave_intervent_to_tidle_trans} \\
		&=&  E_{P}[Y|\mathbf{Z}=\mathbf{z}] + E_{P}[\widetilde{Y}^{(1)}|\widetilde{X}_{1}^{(1)}=\widetilde{x}_{1}^{(1)}]   \nonumber \\
		&& \quad + E_{P}[\widetilde{Y}^{(2)}|\widetilde{X}_{2}^{(2)}=\widetilde{x}_{2}^{(2)}]   \nonumber \\
		&& \qquad  + \cdots + E_{P}[\widetilde{Y}^{(n)}|\widetilde{X}_{n}^{(n)}=\widetilde{x}_{n}^{(n)}]. \label{eq_for_estimate_average _effect}
	\end{eqnarray}	
\end{corollary}
\subsection{Algorithms}

From the theorem \ref{thorem_multi_vals} and its corollary,
we propose a meta-algorithm to estimate the average causal effects of intervention in continuous variables.
The algorithm is divided into two phases; the model building phase and the average causal effect estimating phase.

Algorithm \ref{algo_learn} shows the meta-algorithm of the model building phase. 
In this phase, a original data is used as input, then  models is built while transforming it.
In Algorithm \ref{algo_learn}, ``SupervisedLearn'' represents training of an arbitrary supervised regression algorithm.
``SupervisedPredict'' represents computation of predicted values from a pair of the input data and the models obtained from
training ``SupervisedLearn''.

Algorithm \ref{algo_pred} shows the meta-algorithm of the average causal effect estimating phase.
In this phase, an estimator $\hat{Y}$ of
the average causal effect  $E[Y|do(X_1=x_1), do(X_2=x_2), \ldots, do(X_n=x_n),\mathbf{Z}=\mathbf{z}]$
is computed using models of ``SupervisedLearn'' in Algorithm \ref{algo_learn},
values $x_1, x_2, \ldots, x_n$ and values $\mathbf{z}$,
where $x_1, x_2, \ldots, x_n$ correspond to values of intervention in $\mathbf{X}$ for each individual
and $\mathbf{z}$ correspond to values of each individual's covariates $\mathbf{Z}$.

\begin{figure}
	\begin{algorithm}[H] 
		\caption{Model building phase}\label{algo_learn}
		\begin{algorithmic}[1]
			\Require outcome $Y$, \, treatments $\{X_1, X_2, \ldots, X_n\}$, \, covariates $\mathbf{Z}$
			\Ensure models $\left\{   
			\mathscr{M}_{\mathbf{Z} \rightarrow Y}, \, 
			\mathscr{M}_{\mathbf{Z} \rightarrow X_i}, \, 
			\mathscr{M}_{\widetilde{X}_i^{(i)}\rightarrow \widetilde{Y}^{(i)}}, \,
			\mathscr{M}_{\widetilde{X}_i^{(i)}\rightarrow \widetilde{X}_j^{(i)}} \right\}$	
			\State $\mathscr{M}_{\mathbf{Z} \rightarrow Y} \leftarrow {\mathit SupervisedLearn}(\mathbf{Z}, \: Y)$	
			\State $\widetilde{Y}^{(1)} \leftarrow Y -  {\mathit SupervisedPredict}(\mathscr{M}_{\mathbf{Z} \rightarrow Y}, \: \mathbf{Z})$
			\For {$i$ from $1$ to $n$ }
			\State $\mathscr{M}_{\mathbf{Z} \rightarrow X_i} \leftarrow {\mathit SupervisedLearn}(\mathbf{Z}, \: X_i)$
			\State $\widetilde{X}_i^{(1)} \leftarrow X_i -  {\mathit SupervisedPredict}(\mathscr{M}_{\mathbf{Z} \rightarrow X_i},\: \mathbf{Z})$
			\EndFor
			\For {$i$ from $1$ to $n-1$}
			\State $\mathscr{M}_{\widetilde{X}_i^{(i)} \rightarrow \widetilde{Y}^{(i)}} 
			\leftarrow {\mathit SupervisedLearn}
			(\widetilde{X}_i^{(i)}, \: \widetilde{Y}^{(i)})$
			\State $\widetilde{Y}^{(i+1)} \leftarrow 
			\widetilde{Y}^{(i)} - 
			{\mathit SupervisedPredict}
			(\mathscr{M}_{\widetilde{X}_i^{(i)} \rightarrow \widetilde{Y}^{(i)}},\: \widetilde{X}_i^{(i)})$		
			\For {$j$ from $i+1$ to $n$ }
			\State $\mathscr{M}_{\widetilde{X}_i^{(i)} \rightarrow \widetilde{X}_j^{(i)}} 
			\leftarrow {\mathit SupervisedLearn}
			(\widetilde{X}_i^{(i)}, \: \widetilde{X}_j^{(i)})$
			\State $\widetilde{X}_j^{(i+1)} \leftarrow 
			\widetilde{X}_j^{(i)} - 
			{\mathit SupervisedPredict}
			(\mathscr{M}_{\widetilde{X}_i^{(i)} \rightarrow \widetilde{X}_j^{(i)}},\: \widetilde{X}_i^{(i)})$
			\EndFor
			\EndFor
			\State $\mathscr{M}_{\widetilde{X}_n^{(n)} \rightarrow \widetilde{Y}^{(n)}} 
			\leftarrow {\mathit SupervisedLearn}
			(\widetilde{X}_n^{(n)}, \: \widetilde{Y}^{(n)})$
		\end{algorithmic}
	\end{algorithm}
	%\caption{モデル構築フェーズのアルゴリズム}
	%\ecaption{Algorithm for model building phase}
	\begin{algorithm}[H] 
		\caption{Average causal effect estimating phase}\label{algo_pred}
		\begin{algorithmic}[1]
			\Require 
			\hspace*{\algorithmicindent} \parbox[t]{1.0\linewidth}{
				\raggedright values of intervention $\{x_1, x_2, \ldots, x_n\}$, \, covariates $\mathbf{z}$,\\
				models $\left\{\mathscr{M}_{\mathbf{Z} \rightarrow Y}, \:
				\mathscr{M}_{\mathbf{Z} \rightarrow X_i}, \:
				\mathscr{M}_{\widetilde{X}_i^{(i)}\rightarrow \widetilde{X}_j^{(i)}}, \:
				\mathscr{M}_{\widetilde{X}_i^{(i)}\rightarrow \widetilde{Y}^{(i)}} \right\}$}
			\Ensure the average causal effect $\hat{Y}$
			\For {$i$ from $1$ to $n$ }
			\State $\widetilde{x}_i^{(1)} \leftarrow x_i -  {\mathit SupervisedPredict}(\mathscr{M}_{\mathbf{Z} \rightarrow X_i},\: \mathbf{z})$
			\EndFor
			\For {$i$ from $1$ to $n-1$ }
			\For {$j$ from $i+1$ to $n$ }
			\State $\widetilde{x}_j^{(i+1)} \leftarrow 
			\widetilde{x}_j^{(i)} - 
			{\mathit SupervisedPredict}
			(\mathscr{M}_{\widetilde{X}_i^{(i)} \rightarrow \widetilde{X}_j^{(i)}},\: \widetilde{x}_i^{(i)})$
			\EndFor
			\EndFor
			\State $\hat{Y} \leftarrow {\mathit SupervisedPredict}(\mathscr{M}_{\mathbf{Z} \rightarrow Y},\: \mathbf{z})$
			\For {$i$ from $1$ to $n$ }
			\State $\hat{Y} \leftarrow \hat{Y} + 
			{\mathit SupervisedPredict}(\mathscr{M}_{\widetilde{X}_i^{(i)}\rightarrow \widetilde{Y}^{(i)}}, \: \widetilde{x}_i^{(i)})$
			\EndFor
		\end{algorithmic}
	\end{algorithm}
	%\caption{平均因果効果推定フェーズのアルゴリズム}
	%\ecaption{Algorithm for average causal effect estimating phase}
\end{figure}

\section{Simulation Results and Discussion}\label{sec_num_exam}
We conducted numerical experiments on the proposing method by generating 1000 data sets from
a model such as
$Z \sim \mathcal{U}(0, 1)$, $X \sim sin(Z) + \varepsilon_X$, 
$Y = X \cdot Z  + \varepsilon_Y$, 
$\varepsilon_X \sim \mathcal{U}(-0.5, 0.5)$ and
$\varepsilon_Y \sim \mathcal{N}(0, 0.05)$. 
Here, ``$\sim \mathcal{U}(a, b)$'' represents sampling from
the Uniform distribution on an interval [a, b]. ``$\sim \mathcal{N}(a, b)$''
represents sampling from the Normal distribution
with mean $a$ and standard deviation $b$.

\begin{figure}
	\begin{center}
		\includegraphics[keepaspectratio, scale=0.3]{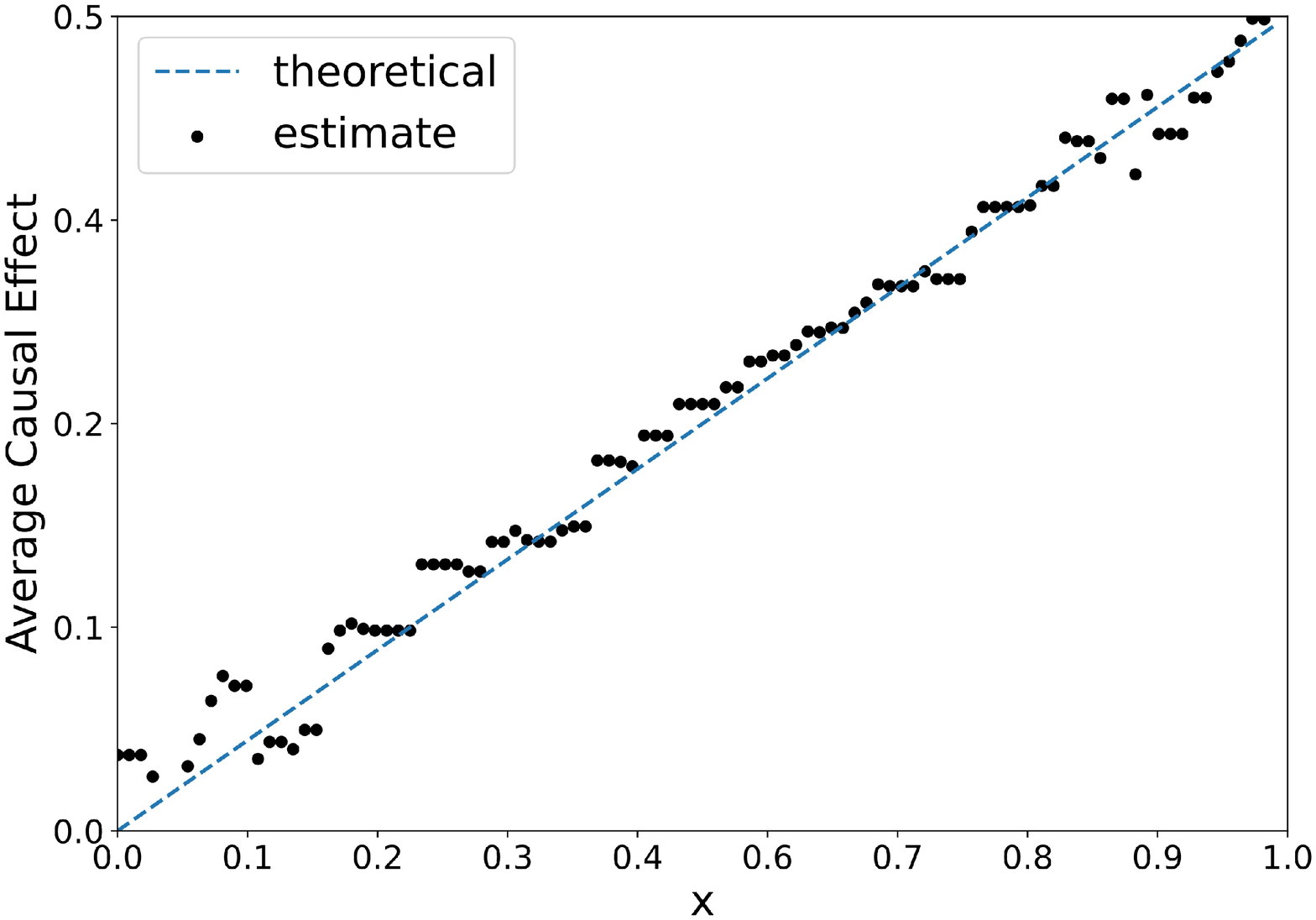}
		\caption{Results of estimating the Average Causal Effect on $Z=0.5$ (LightGBM)}\label{fig_esti_lightgbm}
		\includegraphics[keepaspectratio, scale=0.3]{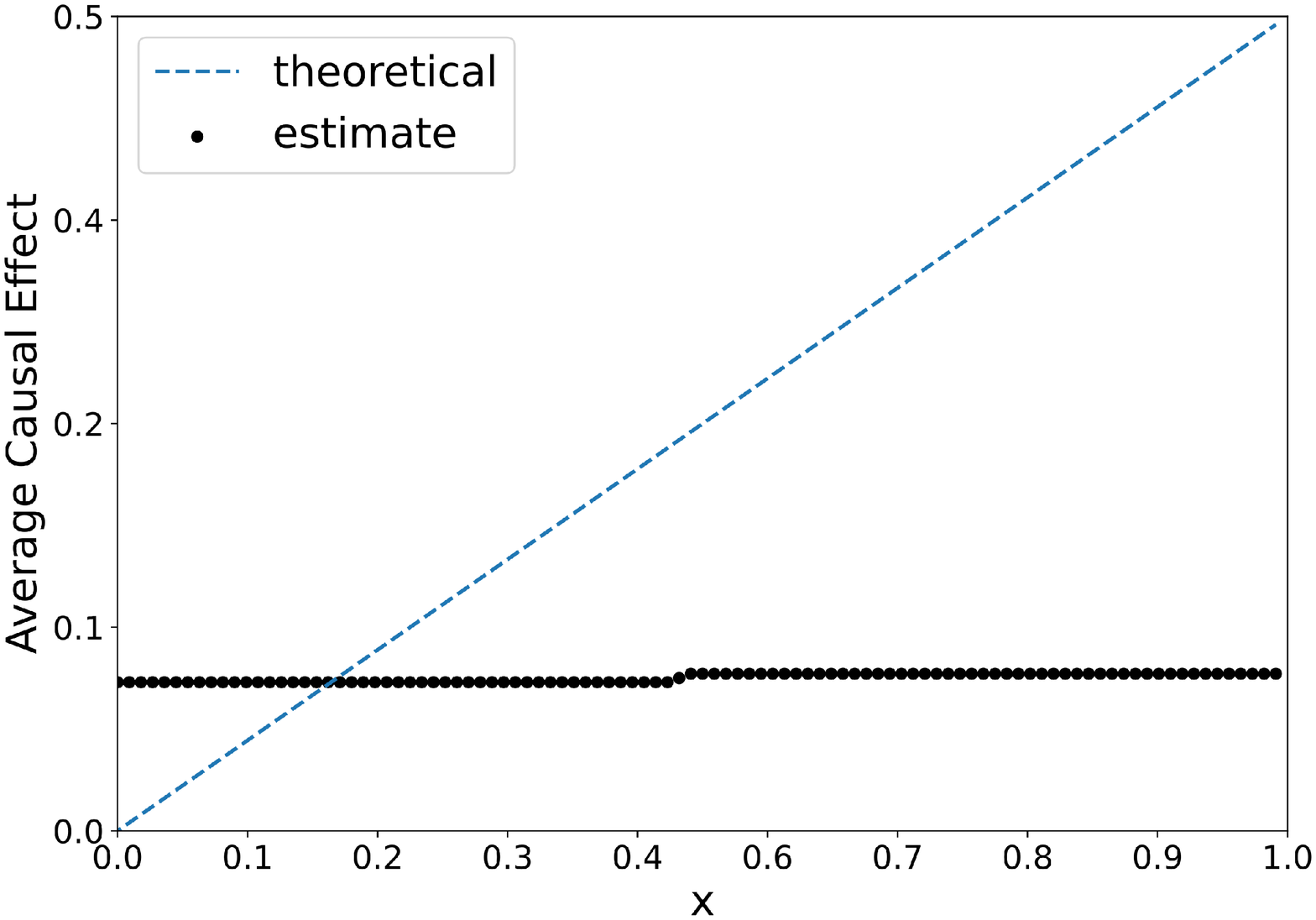}
		\caption{Results of estimating the Average Causal Effect on $Z=0.5$ (RandomForestst)}\label{fig_esti_RF}
	\end{center}
\end{figure}

Two results are shown here.
One is a result of using LightGBM and the other is RandomForest.
Figure \ref{fig_esti_lightgbm} is the result of of using LightGBM.
Figure \ref{fig_esti_RF} is the result of of using RandomForest.
The figures show the estimate or the theoretical values of the average causal effect $E[Y|do(X=x), Z]$,
where $Z$ is fixed to $0.5$.
The plots show the estimated values by the proposing method.
The dashed line shows the theoretical values ($Y = 0.5 X$) for the model which generates data.
In the figures, the $x$-axis corresponds to
values $x$ of intervention, 
and the $y$-axis corresponds to values of the average causal effect.

The result of LightGBM (Figure \ref{fig_esti_lightgbm}) shows that the proposing method 
almost exactly estimates the theoretical values of the average causal effects.
On the other hand, one of RandomForest (Figure \ref{fig_esti_RF}) shows that 
it fails to estimate that values. 

The main reason for this is thought that (a)insufficient hyperparameter search in 
RandomForest and (b)differences in the ability between algorithms to estimate the
conditional expectation.
It is also important to choose an algorithm that can learn well patterns of
data of interest and estimate well the conditional expectation of target.

\section{Conclusion}\label{sec_conclusioin}

We proposed a method to estimate the average causal effects of 
intervention in continuous
variables using supervised regression algorithms of machine learning.
We also have showed that the proposing method preserves for the
identifiability of causal effects for models which generate original data.
In addition, by simulation examines, it has been confirmed that the proposing
method successfully estimates the average causal effect.

This method can estimate values of the average causal effect of intervention in 
continuous variables, whenever we can estimate it from original data. 
Furthermore, because this method is algorithm-free, it can be widely applied to
supervised regression algorithms in machine learning.

\printbibliography

\appendix

\section{Proof of Theorem}

The following $C$-component, $C$-Forest and hedge 
are the important structures of a causal graph for the identifiability of causal effects. 
$C$-component is defined in [\cite{tian2002general}].
$C$-Forest and hedge is first defined by [\cite{shpitser2006identification}].
\begin{definition}[$C$-Component]
	\it
	Let $G$ be a causal graph of semi-Markovian model such that a subset of its bidirected arcs forms a spanning
	tree over all vertices in $G$. 
	Then $G$ is a $C$-Component(confounded Component).	
\end{definition}
\begin{definition}[$C$-Forest]
	\it
	Let $G$ be a causal graph of semi-Markovian model, 
	where $\mathbf{R}$ is the root set. Then $G$ is a $\mathbf{R}$-rooted $C$-Forest if all
	nodes in $G$ form a C-component.
\end{definition}
\begin{definition}[hedge]
	\it
	Let $\mathbf{X}$, $\mathbf{Y}$ be sets of variables in $G$. Let
	$F,F'$ be $\mathbf{R}$-rooted $C$-Forests such that $F\cap \mathbf{X} \neq \phi$, 
	$F'\cap \mathbf{X} = \phi$, $F'\subset F$ and $R \subset An(\mathbf{Y})_{\overline{G}(\mathbf{X})}$.
	Then $F$ and $F'$ form a hedge for $P(\mathbf{Y}|do(\mathbf{X}))$.
\end{definition}

The following theorem gives a necessary and sufficient condition of the identifiability of 
causal effects for joint intervention.
It is presented by [\cite{shpitser2006identification}].
\begin{theorem}[hedge criterion]\label{th_hedge_criterio}
	$P(\mathbf{Y}|do(\mathbf{X}))$ is identifiable from $P$
	in $G$ if and only if there does not exist a hedge for $P(\mathbf{Y}'|do(\mathbf{X}'))$
	in $G$, for any $\mathbf{X}' \subset X$ and $\mathbf{Y}' \subset \mathbf{Y}$. 
\end{theorem}

[\cite{pearl1995causal}] has given $do$-calculus, 
the following rules R1-R3 of transformation between causal effects.
\newline
\underline{$do$-calculus}(\cite{pearl1995causal})
\begin{enumerate}
	\item[R1. ]  $P(\mathbf{Y}|do(\mathbf{X}),\mathbf{Z},\mathbf{W}) 
	= P(\mathbf{Y}|do(\mathbf{X}), \mathbf{W})$
	\ if \ 
	$(\mathbf{Y} \indep \mathbf{Z}|\mathbf{X},\mathbf{W})_{
		\overline{G}(\mathbf{X})}$
	\item[R2. ] $P(\mathbf{Y}|do(\mathbf{X}),do(\mathbf{Z}),\mathbf{W}) 
	= P(\mathbf{Y}|do(\mathbf{X}), \mathbf{Z},\mathbf{W})$ 
	\ if \  $(\mathbf{Y} \indep \mathbf{Z}|\mathbf{X},\mathbf{W})_{
		\overline{\underline{G}}(\mathbf{X}, \mathbf{Z})}$
	\item[R3. ]  $P(\mathbf{Y}|do(\mathbf{X}),do(\mathbf{Z}),\mathbf{W}) 
	= P(\mathbf{Y}|do(\mathbf{X}), \mathbf{W})$ 
	\ if \ $(\mathbf{Y} \indep \mathbf{Z}|\mathbf{X},\mathbf{W})_{
		\overline{G}(\mathbf{X},\mathbf{Z}^*)}$ \\
	where \  $\mathbf{Z}^* = \mathbf{Z} \setminus An(\mathbf{W})_{\overline{G}(\mathbf{X})}$
\end{enumerate}

A necessary and sufficient condition of the identifiability of conditional causal effects
$P(\mathbf{Y}=\mathbf{y}|do(\mathbf{X}=\mathbf{x}),\mathbf{Z}=\mathbf{z})$ is 
obtained by applying R2 above to it, which is presented by [\cite{shpitser2012identification}].
\begin{theorem}[\cite{shpitser2012identification}]\label{shiptiser2012criterion}
	Let $\mathbf{Z}' \subset \mathbf{Z}$ be the maximal set such that
	$P(\mathbf{Y}|do(\mathbf{X}), \mathbf{Z}) =
	P(\mathbf{Y},\mathbf{Z} \setminus \mathbf{Z}'|do(\mathbf{X}),do(\mathbf{Z}'))$.
	Then $P(\mathbf{Y}|do(\mathbf{X}), \mathbf{Z})$ is identifiable in
	$G$ if and only if $P(\mathbf{Y},\mathbf{Z} \setminus \mathbf{Z}'|do(\mathbf{X}),do(\mathbf{Z}'))$
	is identifiable in $G$.
\end{theorem}
Hence, a necessary and sufficient condition of
the identifiability of a conditional causal effect
$P(\mathbf{Y}=\mathbf{y}|do(\mathbf{X}=\mathbf{x}),\mathbf{Z}=\mathbf{z})$ 
is that the $P(\mathbf{Y},\mathbf{Z} \setminus \mathbf{Z}'|do(\mathbf{X}),do(\mathbf{Z}'))$ above
satisfies the hedge criterion. 

\subsection{Lemmas}

First, we give some lemmas to prove the theorems.
\begin{lemma}\label{lemma_intervention_eq}
	Let $G$ be a DAG for $\mathbf{V}$ and  $\mathbf{U}$. 
	%where $\mathbf{V}$ and $\mathbf{U}$ are observed and unobserved variables each others, 
	For disjoint sets of $\mathbf{X}, \mathbf{Y}, \mathbf{Z} \subset \mathbf{V}$, 
	let $\mathbf{X} \subset An(\mathbf{Y})_G$, and
	let $P(\mathbf{Y}|do(\mathbf{X}),\mathbf{Z})$ be identifiable in $G$.
	Let	$\mathbf{Z}_{De} = \mathbf{Z} \cap De(\mathbf{X})_G \cap De(\mathbf{Y})_G$.
	Then, 
	\begin{eqnarray}
		\lefteqn{P(\mathbf{Y} |do(\mathbf{X}),\mathbf{Z})} \quad \nonumber\\
		&=&  
		\begin{cases}
			P(\mathbf{Y} |\mathbf{X},\mathbf{Z}) & \text{if $\mathbf{Z}_{De}=\phi$} \\
			\frac{P(\mathbf{Y} |\mathbf{X}, \mathbf{Z}\setminus \mathbf{Z}_{De})
				P(\mathbf{Z}_{De}|\mathbf{Y},\mathbf{X},\mathbf{Z}\setminus\mathbf{Z}_{De})}{P(\mathbf{Z}_{De}|\mathbf{X}, \mathbf{Z}\setminus\mathbf{Z}_{De})} & \text{if $\mathbf{Z}_{De}\neq\phi$}.
		\end{cases} \label{eq_prob_formula_intervent_first_lemma}
	\end{eqnarray}
\end{lemma}

\begin{proof}[proof of Lemma \ref{lemma_intervention_eq}]
	%以下，各 $Z_i \in \mathbf{Z}$ において $\mathbf{X}$ 及び $\mathbf{Y}$ と確率的に独立ではないとして議論を行う．
	In this poof, each $Z_i \in \mathbf{Z}$ is assumed  not independent of $\mathbf{Y}$.  
	If there exists $Z_i \in \mathbf{Z}$ such that it is independent of $\mathbf{Y}$,
	remove such $\mathbf{Z}_i$'s by using $do$-calculus R1 to them,
	then we can prove that case in the same manner bellow.
	
	Firstly, note that $\mathbf{Z}\cap An(\mathbf{X})_{G}\cap De(\mathbf{Y})_{G}=\phi$.
	Let $\mathbf{V}'=\mathbf{V}\setminus(\mathbf{Y}\cup\mathbf{X}\cup\mathbf{Z})$.
	If $\mathbf{Z}\cap An(\mathbf{X})_{G}\cap De(\mathbf{Y})_{G} \neq \phi$,
	then there exists a directed path through $\mathbf{V}'$ 
	such that $\mathbf{Y} \dotarrow{\mathbf{V}'} \mathbf{Z} \dotarrow{\mathbf{V}'}  \mathbf{X}$,
	which contradicts the assumption $\mathbf{X} \subset An(\mathbf{Y})_G$.
	Thus, let disjoint sets $\mathbf{Z}_1$, $\mathbf{Z}_2$ and $\mathbf{Z}_3$ be
	\begin{eqnarray}
		%\mathbf{Z}_1 &=& \mathbf{Z} \cap An(\mathbf{X})_G \cap An(\mathbf{Y})_G \\
		\mathbf{Z}_1 &=& (\mathbf{Z} \setminus De(\mathbf{X})_G) \cap An(\mathbf{Y})_G \nonumber \\
		\mathbf{Z}_2 &=& \mathbf{Z} \cap De(\mathbf{X})_G \cap An(\mathbf{Y})_G \nonumber \\
		\mathbf{Z}_3 &=& \mathbf{Z} \cap De(\mathbf{X})_G \cap De(\mathbf{Y})_G, \nonumber
	\end{eqnarray}
	then $\mathbf{Z}$ can be divided such that $\mathbf{Z}=\mathbf{Z}_1\cup\mathbf{Z}_2\cup\mathbf{Z}_3$.
	
	Secondly, note that there exist the directed paths through $\mathbf{V}'$ 
	between $\mathbf{Z}_1$, $\mathbf{Z}_2$ and $\mathbf{Z}_3$ 
	in the only 3 cases as bellow.
	\begin{enumerate}
		\item[P1. ] $\mathbf{Z}_1 \dotarrow{\mathbf{V}'} \mathbf{Z}_2$
		\item[P2. ] $\mathbf{Z}_2 \dotarrow{\mathbf{V}'} \mathbf{Z}_3$
		\item[P3. ] $\mathbf{Z}_1\dotarrow{\mathbf{V}'} \mathbf{Z}_3$
	\end{enumerate}
	In fact, as to P1 for example, if there exist directed paths through
	$\mathbf{Z}_2 \dotarrow{\mathbf{V}'} \mathbf{Z}_1$, 
	then there exists a path such that 
	$X_i \dotarrow{\mathbf{V}'}  \mathbf{Z}_2 \dotarrow{\mathbf{V}'}  \mathbf{Z}_1\dotarrow{\mathbf{V}'} X_j$,
	which contradicts the assumption $\mathbf{Z}_1 \subset \mathbf{Z} \setminus De(\mathbf{X})_G$.
	Similarly, the other paths than listed above are denied.
	
	Therefore, we can obtain 
	$P(\mathbf{Y}, \mathbf{X}, \mathbf{Z}) = \sum_{\mathcal{X}_{\mathbf{V}'}} P(\mathbf{V})$ 
	as the form
	\begin{eqnarray}
		P(\mathbf{Y}, \mathbf{X}, \mathbf{Z}) &=& P(\mathbf{Y}|\mathbf{X}, \mathbf{Z}_1, \mathbf{Z}_2) 
		\cdot P(\mathbf{X}|\mathbf{Z}_1) \cdot P(\mathbf{Z}_1) \nonumber\\
		&& \times P(\mathbf{Z}_2|\mathbf{X}, \mathbf{Z}_1) 
		\cdot P(\mathbf{Z}_3|\mathbf{Y}, \mathbf{X},\mathbf{Z}_1 ,\mathbf{Z}_2). \nonumber
	\end{eqnarray}
	By (\ref{def_causal_effect_no_condi_eq}), we have
	\begin{eqnarray}
		P(\mathbf{Y}, \mathbf{Z}|do(\mathbf{X})) &=& P(\mathbf{Y}|\mathbf{X}, \mathbf{Z}_1, \mathbf{Z}_2)
		\cdot  P(\mathbf{Z}_1) \cdot P(\mathbf{Z}_2|\mathbf{X}, \mathbf{Z}_1) \nonumber\\
		&& \times P(\mathbf{Z}_3|\mathbf{Y}, \mathbf{X},\mathbf{Z}_1 ,\mathbf{Z}_2). \label{eq_do_x_p_y_z}
	\end{eqnarray}
	In the case $\mathbf{Z}_3 = \phi$, we obtain by marginalizing out $\mathbf{Y}$ of (\ref{eq_do_x_p_y_z}),
	\begin{eqnarray}
		P(\mathbf{Z}|do(\mathbf{X})) &=& \sum_{\mathbf{y} \in \mathcal{X}_{\mathbf{Y}}}
		P(\mathbf{Y}=\mathbf{y}, \mathbf{Z}|do(\mathbf{X})) \nonumber \\
		&=& \sum_{\mathbf{y} \in \mathcal{X}_{\mathbf{Y}}} P(\mathbf{Y}=\mathbf{y}|\mathbf{X}, \mathbf{Z}_1, \mathbf{Z}_2) \cdot  P(\mathbf{Z}_1) \cdot P(\mathbf{Z}_2|\mathbf{X}, \mathbf{Z}_1) \nonumber\\
		&=& P(\mathbf{Z}_1) \cdot P(\mathbf{Z}_2|\mathbf{X}, \mathbf{Z}_1) 
		\sum_{\mathbf{y} \in \mathcal{X}_{\mathbf{Y}}} P(\mathbf{Y}=\mathbf{y}|\mathbf{X}, \mathbf{Z}_1, \mathbf{Z}_2) \nonumber\\
		&=& P(\mathbf{Z}_1) \cdot P(\mathbf{Z}_2|\mathbf{X}, \mathbf{Z}_1). \nonumber
	\end{eqnarray}
	On the other hand, in the case that $\mathbf{Z}_3 \neq \phi$, we obtain
	\begin{eqnarray}
		P(\mathbf{Z}|do(\mathbf{X})) &=& \sum_{\mathbf{y} \in \mathcal{X}_{\mathbf{Y}}}
		P(\mathbf{Y}=\mathbf{y}, \mathbf{Z}|do(\mathbf{X})) \nonumber\\
		&=& \sum_{\mathbf{y} \in \mathcal{X}_{\mathbf{Y}}} P(\mathbf{Y}=\mathbf{y}|\mathbf{X}, \mathbf{Z}_1, \mathbf{Z}_2) \cdot  P(\mathbf{Z}_1) \nonumber\\
		&& \times P(\mathbf{Z}_2|\mathbf{X}, \mathbf{Z}_1) \cdot P(\mathbf{Z}_3|\mathbf{Y}=\mathbf{y}, \mathbf{X},\mathbf{Z}_1 ,\mathbf{Z}_2) \nonumber\\
		&=& P(\mathbf{Z}_1) \cdot P(\mathbf{Z}_2|\mathbf{X}, \mathbf{Z}_1) 
		\sum_{\mathbf{y} \in \mathcal{X}_{\mathbf{Y}}}  P(\mathbf{Z}_3|\mathbf{Y}=\mathbf{y}, \mathbf{X},\mathbf{Z}_1 ,\mathbf{Z}_2) \nonumber\\
		&& \times P(\mathbf{Y}=\mathbf{y}|\mathbf{X}, \mathbf{Z}_1, \mathbf{Z}_2) \nonumber\\
		&=& P(\mathbf{Z}_1) \cdot P(\mathbf{Z}_2|\mathbf{X}, \mathbf{Z}_1) \cdot P(\mathbf{Z}_3|\mathbf{X},\mathbf{Z}_1 ,\mathbf{Z}_2). \nonumber
	\end{eqnarray}
	Summarizing the above results,
	\begin{equation}
		P(\mathbf{Z}|do(\mathbf{X})) = 
		\begin{cases}
			P(\mathbf{Z}_1) \cdot P(\mathbf{Z}_2|\mathbf{X}, \mathbf{Z}_1)  & \text{if $\mathbf{Z}_3=\phi$} \\
			P(\mathbf{Z}_1) \cdot P(\mathbf{Z}_2|\mathbf{X}, \mathbf{Z}_1) \cdot P(\mathbf{Z}_3|\mathbf{X},\mathbf{Z}_1 ,\mathbf{Z}_2) & \text{if $\mathbf{Z}_3 \neq \phi$}.
		\end{cases} \label{eq_do_x_p_z}
	\end{equation}
	Inserting (\ref{eq_do_x_p_y_z}) and (\ref{eq_do_x_p_z}) into  (\ref{def_causal_effect_on_condi_eq}),
	\begin{eqnarray}
		P(\mathbf{Y}|do(\mathbf{X}), \mathbf{Z}) &=& \frac{P(\mathbf{Y}, \mathbf{Z}|do(\mathbf{X}))}{P(\mathbf{Z}|do(\mathbf{X}))} \nonumber \\
		&=& 
		\begin{cases}
			P(\mathbf{Y}|\mathbf{X}, \mathbf{Z}_1, \mathbf{Z}_2) & \text{if $\mathbf{Z}_3=\phi$} \nonumber\\
			\frac{P(\mathbf{Y}|\mathbf{X}, \mathbf{Z}_1, \mathbf{Z}_2) P(\mathbf{Z}_3|\mathbf{Y}, \mathbf{X},\mathbf{Z}_1 ,\mathbf{Z}_2)}
			{P(\mathbf{Z}_3|\mathbf{X},\mathbf{Z}_1 ,\mathbf{Z}_2)} & \text{if $\mathbf{Z}_3 \neq \phi$} \nonumber.
		\end{cases}\nonumber
	\end{eqnarray}
	Note that $\mathbf{Z}_{De}=\mathbf{Z}_3$ and $\mathbf{Z} \setminus \mathbf{Z}_{De} = \mathbf{Z}_1 \cup \mathbf{Z}_2$, 
	therefore we obtain (\ref{eq_prob_formula_intervent_first_lemma}).
\end{proof}

\begin{lemma}\label{lemma_edge_del_modify_between_x_also_identifibility_preserve}
	Let $G$ be a DAG for $\mathbf{V}$ and $\mathbf{U}$. 
	%where $\mathbf{V}$ and $\mathbf{U}$ are observed and unobserved variables each others, 
	For disjoint sets of $\mathbf{X}, \mathbf{Y} \subset \mathbf{V}$, let 
	$P(\mathbf{Y}|do(\mathbf{X}))$ be identifiable in $G$. 
	For fixed $X' \in \mathbf{X}$, 
	let $\widetilde{G}$ a graph obtained from $G$ by deleting all the arrows emerging from
	$X'$ to $\mathbf{V} \setminus \{X'\}$  and 
	all the arrows emerging from $X'$ to $\mathbf{U}$.
	Then $P(\mathbf{Y}|do(\mathbf{X}\setminus \{X'\}))$ is also identifiable in $\widetilde{G}$.
\end{lemma}

\begin{proof}[proof of Lemma \ref{lemma_edge_del_modify_between_x_also_identifibility_preserve}]
	Let $P$ be a model relative to $G$ and let $\widetilde{P}$ be a model relative to $\widetilde{G}$.
	Assume that $P(\mathbf{Y}|do(\mathbf{X}))$ is identifiable in $G$ but
	that $\widetilde{P}(\mathbf{Y}|do(\mathbf{X}\setminus \{X'\}))$ is not identifiable in $\widetilde{G}$.
	
	By theorem \ref{th_hedge_criterio},
	for some $\mathbf{Y}_0 \subset \mathbf{Y}$, $\mathbf{X}_0\subset (\mathbf{X}\setminus \{X'\})$ 
	and $\mathbf{R} \subset An(\mathbf{Y}_0)_{\overline{\widetilde{G}}(\mathbf{X}_0)}$, 
	there exist $\mathbf{R}$-rooted $C$-Forests $\widetilde{F}, \widetilde{F}'$ such that
	$\widetilde{F} \cap \mathbf{X}_0 \neq \phi$ and $\widetilde{F}' \cap \mathbf{X}_0 = \phi$, 
	and $\widetilde{F},\widetilde{F}'$ form
	a hedge for $\widetilde{P}(\mathbf{Y}_0|do(\mathbf{X}_0))$.
	
	Now let $F=\widetilde{F}\setminus \{X'\}$ and $F'=\widetilde{F}' \setminus  \{X'\}$,
	then $F$ and $F'$ form a hedge for $P(\mathbf{Y}_0|do(\mathbf{X}_0))$ in $G$.
	
	In fact,  $\mathbf{R} \subset \mathbf{Y}_0$
	and $F \cap \mathbf{X}_0 \neq \phi$ and
	$F' \cap \mathbf{X}_0 = \phi$.
	In addition, $G$ has the same arrows as $\widetilde{G}$ 
	between variables of $\mathbf{V}\setminus \{X'\}$
	and between $\mathbf{U}$ and $\mathbf{V}\setminus \{X'\}$.
	Since $F,F' \subset \mathbf{V}\setminus \{X'\}$,
	$F$ and $F'$ have same bidirected arcs and same arrows in $G$ as in $\widetilde{G}$.
	Thus, $F$ and $F'$ are $C$-Forest in $G$, and form
	a hedge for $P(\mathbf{Y}_0|do(\mathbf{X}_0))$.
	Therefore, $P(\mathbf{Y}|do(\mathbf{X}))$ is
	not identifiable in $G$, which contradicts the assumption.
	Consequently, $\widetilde{P}(\mathbf{Y}|do(\mathbf{X}\setminus \{X'\}))$ is identifiable in $\widetilde{G}$.
	
	This completes the proof of the lemma.	
\end{proof}

\begin{lemma}\label{lemma_edge_del_modify_between_z_also_identifibility_preserve}
	Let $G$ be a DAG for $\mathbf{V}$ and $\mathbf{U}$. 
	%where $\mathbf{V}$ and $\mathbf{U}$ are observed and unobserved variables each others, 
	For disjoint sets of $\mathbf{X}, \mathbf{Y}, \mathbf{Z} \subset \mathbf{V}$, let 
	$P(\mathbf{Y}|do(\mathbf{X}),\mathbf{Z})$  be identifiable in $G$.
	Let $\widetilde{G}$ a graph obtained from $G$ by deleting all the arrows emerging from
	$\mathbf{Z}$ to  $\mathbf{V}\setminus\mathbf{Z}$ and 
	all the arrows emerging from $\mathbf{Z}$ to $\mathbf{U}$.
	Then, $P(\mathbf{Y}|do(\mathbf{X}),\mathbf{Z})$ is also identifiable in $\widetilde{G}$.
\end{lemma}

\begin{proof}[proof of Lemma \ref{lemma_edge_del_modify_between_z_also_identifibility_preserve}]
	Let $P$ be a model relative to $G$ and let $\widetilde{P}$ be a model relative to $\widetilde{G}$.
	Assume that $P(\mathbf{Y}|do(\mathbf{X}),\mathbf{Z})$ is identifiable in $G$ but
	that $\widetilde{P}(\mathbf{Y}|do(\mathbf{X}))$ is not identifiable in $\widetilde{G}$.
	
	By theorem \ref{shiptiser2012criterion}, there exists $\mathbf{Z}_1\subset\mathbf{Z}$ such that
	\[
	P(\mathbf{Y}|do(\mathbf{X}),\mathbf{Z}) = P(\mathbf{Y},\mathbf{Z}\setminus\mathbf{Z}_1|do(\mathbf{X}), do(\mathbf{Z}_1)),
	\]
	and $P(\mathbf{Y},\mathbf{Z}\setminus\mathbf{Z}_1|do(\mathbf{X}), do(\mathbf{Z}_1))$ is identifiable in $G$. 
	On the other hand, by theorem \ref{th_hedge_criterio},
	for some $\mathbf{Y}'\subset \mathbf{Y}$, $\mathbf{X}'\subset \mathbf{X}$ and $\mathbf{R} \subset An(\mathbf{Y'})_{\overline{\widetilde{G}}(\mathbf{X}')}$, 
	there exist $\mathbf{R}$-rooted $C$-Forests $\widetilde{F}, \widetilde{F}'$ such that
	$\widetilde{F} \cap \mathbf{X}' \neq \phi$ and $\widetilde{F}' \cap \mathbf{X}' = \phi$,
	and $\widetilde{F}$,$\widetilde{F}'$ form
	a hedge for $\widetilde{P}(\mathbf{Y}'|do(\mathbf{X}'))$.
	
	Now, let $F=\widetilde{F}\setminus \mathbf{Z}_1$ and $F'=\widetilde{F}' \setminus \mathbf{Z}_1$.
	Then $F$ and $F'$ form a hedge for $P(\mathbf{Y},\mathbf{Z}\setminus\mathbf{Z}_1|do(\mathbf{X}), do(\mathbf{Z}_1))$
	in $G$.
	In fact, $\mathbf{R} \subset \mathbf{Y}'$ and $F \cap \mathbf{X}' \neq \phi$ and
	$F' \cap \mathbf{X}' = \phi$.
	In addition,
	$G$ has the same arrows as $\widetilde{G}$ between variables of $\mathbf{V}\setminus \mathbf{Z}_1$
	and between $\mathbf{U}$ and $\mathbf{V}\setminus \mathbf{Z}_1$.
	Since $F,F' \subset (\mathbf{V}\setminus \mathbf{Z}_1)$, $F$ and $F'$ have same bidirected arcs and same arrows in $G$ as in $\widetilde{G}$.
	Thus, $F$ and $F'$ are $C$-Forest in $G$, and form
	a hedge for $P(\mathbf{Y}'|do(\mathbf{X}'))$.		
	Therefore, $P(\mathbf{Y},\mathbf{Z}\setminus\mathbf{Z}_1|do(\mathbf{X}), do(\mathbf{Z}_1))$ is
	not identifiable in $G$, which contradicts the assumption.
	Consequently, $\widetilde{P}(\mathbf{Y}|do(\mathbf{X}))$ is identifiable in $\widetilde{G}$.
	
	By the way, by the definition of $\widetilde{G}$, it holds that $(\mathbf{Y} \indep \mathbf{Z})_{\widetilde{G}}$.
	Thus, $(\mathbf{Y} \indep \mathbf{Z}|\mathbf{X})_{\widetilde{G}}$.
	Therefore, We can apply $do$-calculus R1 to $\widetilde{P}(\mathbf{Y}|do(\mathbf{X}),\mathbf{Z})$, 
	and we obtain
	\[
	\widetilde{P}(\mathbf{Y}|do(\mathbf{X}),\mathbf{Z})=\widetilde{P}(\mathbf{Y}|do(\mathbf{X})).
	\]
	That is, the identifiability of $\widetilde{P}(\mathbf{Y}|do(\mathbf{X}),\mathbf{Z})$ coincides
	that of $\widetilde{P}(\mathbf{Y}|do(\mathbf{X}))$.
	As a result, $\widetilde{P}(\mathbf{Y}|do(\mathbf{X}),\mathbf{Z})$ is identifiable in $\widetilde{G}$.
	
	This completes the proof of the lemma.
\end{proof}

\begin{lemma}\label{lemma_one_do_minus_modify_also_identifibility_preserve}
	Let $G$ be a DAG for $\mathbf{V}$ and $\mathbf{U}$. 
	Assume that for disjoint sets $\mathbf{Y},\mathbf{X}\subset \mathbf{V}$,
	$\mathbf{X} \subset An(\mathbf{Y})_G\cup Ind(\mathbf{Y})_{G}$.
	Let a $X' \in \mathbf{X}$ be fixed.
	For $A\in \mathbf{V}\setminus\{X'\}\cup\mathbf{U}$, let
	\begin{equation}
		\widetilde{A} = 
		\begin{cases}
			A - E[A|X'=x'] & \text{if $A \neq X'$}\\
			A                              & \text{if $A = X'$},
		\end{cases} \label{eq_trans_lemma_one_do_minus_modify}
	\end{equation}
	For $\mathbf{A} \subset \mathbf{V}\cup\mathbf{U}$, 
	let $\widetilde{\mathbf{A}} = \{\widetilde{A}|A\in\mathbf{A}\}$.
	Then, $\widetilde{\mathbf{V}} \cup \widetilde{\mathbf{U}}$ is a semi-Markovian model.
	Let $\widetilde{G}$ be a DAG for $\widetilde{\mathbf{V}}\cup\widetilde{\mathbf{U}}$, then
	$(\widetilde{\mathbf{X}} \setminus\{X'\}) \subset 
	An(\widetilde{\widetilde{\mathbf{Y}}})_{\widetilde{G}}\cup Ind(\widetilde{\mathbf{Y}})_{\widetilde{G}}$.
	Moreover, 
	if $P(\mathbf{Y}|do(\mathbf{X}))$ is identifiable in $G$,
	then $P(\widetilde{\mathbf{Y}}|do(\widetilde{\mathbf{X}}\setminus\{X'\}))$
	also is identifiable in $\widetilde{G}$.
\end{lemma}

\begin{proof}[proof of Lemma \ref{lemma_one_do_minus_modify_also_identifibility_preserve}]
	Initially, we will show the lemma in the case that $\mathbf{X} \cap Ind(\mathbf{Y})_G = \phi$.	
	Let $\mathbf{V}_0 = \mathbf{V}\setminus\{X'\}$.
	Note that, for any $V_0 \in \mathbf{V}_0$,
	\begin{equation}
		\widetilde{V}_0 \indep X' \label{eq_indep_wx_z_x}.
	\end{equation}
	In fact, for any variables $\xi$ and $\eta$, it holds that
	\begin{equation}
		\xi \indep \eta \quad \Longleftrightarrow \quad E[\xi|\eta] = E[\xi], \nonumber
	\end{equation}
	and for $\widetilde{V}_0$ and $X'$,
	\begin{eqnarray}
		E[\widetilde{V}_0|X'] &=& E[V_0-E[V_0|X']|X'] \nonumber\\
		&=& E[V_0|X'] - E[V_0|X'] \nonumber \\
		&=& 0  = E[\widetilde{V}_0]. \nonumber
	\end{eqnarray}
	Thus, 
	\begin{equation}
		P(\widetilde{V}_0,  X') = P(\widetilde{V}_0)P( X'). \nonumber
	\end{equation}
	Similarly, for $U\in \mathbf{U}$,
	\begin{equation}
		P(\widetilde{U}, X') = P(\widetilde{U})P(X').  \nonumber
	\end{equation}
	By the way, for $v \in \mathcal{X}_{V_0}$ and $u \in \mathcal{X}_{U}$, let
	\[
	\widetilde{v} = v - E[V_0|X'=x'], \quad \widetilde{u} = u - E[U|X'=x'],
	\]
	then according to the definition of $\widetilde{V}_0$ and $\widetilde{U}$,
	\begin{equation}
		P(\widetilde{V}_0=\widetilde{x}, \widetilde{U}=\widetilde{u}) = P(V_0=v, U=u|X'=x'). \nonumber
	\end{equation}
	Therefore, 
	\begin{eqnarray}
		P(\widetilde{V}_0=\widetilde{v}, \widetilde{U}=\widetilde{u}, \widetilde{X}'=\widetilde{x}') &=&
		P(\widetilde{V}_0=\widetilde{v}, \widetilde{U}=\widetilde{u}|\widetilde{X}'=\widetilde{x}') P(\widetilde{X}'=\widetilde{\mathbf{z}}) \nonumber\\
		&=& P(\widetilde{V}_0=\widetilde{v}, \widetilde{U}=\widetilde{u}) P(\widetilde{X}'=\widetilde{x}') \nonumber\\
		&=&  P(V_0=v, U=u|X'=x') P(X'=x') \nonumber\\
		&=&  P(V_0=v, U=u, X'=x'). \nonumber
	\end{eqnarray}
	Similarly, for $\mathbf{V}' \subset \mathbf{V}_0$, $\mathbf{U}' \subset \mathbf{U}$,
	$\mathbf{v} \in \mathcal{X}_{\mathbf{V}'}$ and $\mathbf{u} \in \mathcal{X}_{\mathbf{U}'}$,
	let 
	\[
	\widetilde{\mathbf{v}} = \mathbf{v} - E[\mathbf{V}'|X'], \quad \widetilde{\mathbf{u}} = \mathbf{u} - E[\mathbf{U}'|X'],
	\]
	then, 
	\begin{equation}
		P(\mathbf{V}'=\mathbf{v}, \mathbf{U}'=\mathbf{u}, X'=x') = 
		P(\widetilde{\mathbf{X}}'=\widetilde{\mathbf{v}}, \widetilde{\mathbf{U}}'=\widetilde{\mathbf{u}}, X'=x').  \nonumber
	\end{equation}
	Therefore
	\begin{eqnarray}
		\lefteqn{P(V=v|Pa(V)_G=pa_V, UPa(V)_G=u_V)} \quad \nonumber\\
		&=&  P(\widetilde{V}=\widetilde{v}|\widetilde{Pa}(V)_G=\widetilde{pa_V}, \widetilde{UPa}(V)_G=\widetilde{u_V}),  \label{eq_preserve_parents_x}
	\end{eqnarray}
	where $pa_V$ and $u_V$ are values of $Pa(V)$ and $UPa(V)$. \\
	Hence, 
	\begin{equation}
		Pa(\widetilde{V})_{\widetilde{G}} \subset \widetilde{Pa}(V)_G \label{eq_indep_pa_tv_include_tilde_pav_x}
	\end{equation}
	\begin{equation}
		UPa(\widetilde{V})_{\widetilde{G}} \subset \widetilde{UPa}(V)_G. \label{eq_indep_upa_tv_include_tilde_upav_x}
	\end{equation}
	Now, for $V \in \mathbf{X}$, $W\in Pa(V)_G$ and $U\in UPa(V)_G$,
	the following equations
	(\ref{at_proof_lemma_condi_minus_in_case_v_in_z_w_nin_z_x})-(\ref{at_proof_lemma_condi_minus_in_case_v_nin_z_then_u_in_pav_x})
	hold.
	\begin{equation}
		V =X' \quad and \quad W \neq X' \quad \Longrightarrow \quad  \widetilde{V} \indep \widetilde{W} 
		\label{at_proof_lemma_condi_minus_in_case_v_in_z_w_nin_z_x}
	\end{equation}
	\begin{equation}
		V \neq X' \quad and \quad W = X' \quad \Longrightarrow \quad \widetilde{V} \indep \widetilde{W}
		\label{at_proof_lemma_condi_minus_in_case_v_nin_z_w_in_z_x}
	\end{equation}
	\begin{equation}
		V \neq X'  \quad and \quad W \neq X' \quad \Longrightarrow \quad \widetilde{W} \in Pa(\widetilde{V})_{\widetilde{G}}
		\label{at_proof_lemma_condi_minus_in_case_v_nin_z_w_nin_z_x}
	\end{equation}
	\begin{equation}
		V = X' \quad \Longrightarrow \quad  \widetilde{V} \indep \widetilde{U} 
		\label{at_proof_lemma_condi_minus_in_case_v_in_z_the_idp_u_x}
	\end{equation}
	\begin{equation}
		V \neq X' \quad \Longrightarrow \quad \widetilde{U} \in UPa(\widetilde{V})_{\widetilde{G}}.
		\label{at_proof_lemma_condi_minus_in_case_v_nin_z_then_u_in_pav_x}
	\end{equation}
	(\ref{at_proof_lemma_condi_minus_in_case_v_in_z_w_nin_z_x}),
	(\ref{at_proof_lemma_condi_minus_in_case_v_nin_z_w_in_z_x}) and
	(\ref{at_proof_lemma_condi_minus_in_case_v_in_z_the_idp_u_x}) hold from (\ref{eq_indep_wx_z_x}).
	For (\ref{at_proof_lemma_condi_minus_in_case_v_nin_z_w_nin_z_x}),
	assume that $V \neq X'$ and $W \neq X'$ and $\widetilde{W} \notin Pa(\widetilde{V})_G$.
	From (\ref{eq_indep_pa_tv_include_tilde_pav_x}),  $\widetilde{W} \in \widetilde{Pa}(V) \setminus Pa(\widetilde{V})_{\widetilde{G}}$.
	Thus, at the right hand side of (\ref{eq_preserve_parents_x}), 
	\begin{equation}
		P(\widetilde{V}|\widetilde{Pa}(V)_G, \widetilde{UPa}(V)_G) =
		P(\widetilde{V}|\widetilde{Pa}(V)_G\setminus \widetilde{W}, \widetilde{UPa}(V)_G). \nonumber
	\end{equation}	
	Therefore,
	\begin{equation}
		\widetilde{V} \indep \widetilde{W}. \label{eq_indep_tv_tw_x} 
	\end{equation}
	Now, by (\ref{eq_indep_wx_z_x}), it holds that $\widetilde{V} \indep X'$ and
	$\widetilde{W} \indep X'$.
	Thus, 
	\[
		\widetilde{V} \indep \widetilde{W} + E[W|X'] = W,
	\]
	and
	\[
		\widetilde{W} \indep \widetilde{V} + E[V|X'] = V.
	\]
	Therefore, since
	\begin{equation*}
		 E[\widetilde{V}+E[V|X']|\widetilde{W},X'] = E[\widetilde{V}|\widetilde{W},X'] - E[E[V|X']|\widetilde{W},X'] = E[V|X'],
	\end{equation*}
	it holds that
	\begin{eqnarray*}
		E[V|W] &=&  E[\widetilde{V}+E[V|X']|\widetilde{W}+E[W|X'], \widetilde{W}, X']\\
		&=& E[E[V|X']|\widetilde{W}+E[W|X'], \widetilde{W}, X'] \\
		&=& E[V|X']
	\end{eqnarray*}
	Since any variables $A$, $B$ and $C$, if $E[A|B] = E[A|C]$, then
	there exists a function $f:\mathcal{X}_B \rightarrow \mathcal{X}_C$ such that 
	\begin{equation*}
		  P(B=b) = P(C=f(b)),
	\end{equation*}
	there exists a funcition $f:\mathcal{X}_{X'} \rightarrow \mathcal{X}_{W}$ such that $W=f(X')$.
	Therefore, it contradicts that $W\neq X'$ because we assume that $V \neq f(\mathbf{W})$ for any 
	function $f$, $V \in \mathbf{V}$ and $\mathbf{W}\subset \mathbf{V} \setminus V$.
	for any function $f$ and any $\mathbf{W} $ at the beginning of this paper.
	Therefore, (\ref{at_proof_lemma_condi_minus_in_case_v_nin_z_w_nin_z_x}) holds.
	For (\ref{at_proof_lemma_condi_minus_in_case_v_nin_z_then_u_in_pav_x}), it is similarly shown 
	by replacing  $Pa(V)_G$ with $UPa(V)_G$ and replacing $Pa(\widetilde{V})_G$ with $UPa(\widetilde{V})_G$
	at the above proof of (\ref{at_proof_lemma_condi_minus_in_case_v_nin_z_w_nin_z_x})
	because always $U \neq X'$.
	
	Now, we consider a DAG $G'$ obtained from G by replacing each
	$A\in \mathbf{V}\cup\mathbf{U}$ with $\widetilde{A}$.
	Since $P(\mathbf{Y}|do(\mathbf{X}))$ is identifiable in $G$, 
	$P(\widetilde{\mathbf{Y}}|do(\widetilde{\mathbf{X}}))$
	is identifiable in $G'$.
	Because of 	(\ref{at_proof_lemma_condi_minus_in_case_v_in_z_w_nin_z_x})-(\ref{at_proof_lemma_condi_minus_in_case_v_nin_z_then_u_in_pav_x}),
	$\widetilde{G}$ is just a graph obtained by deleting the arrows in $G'$
	between $X'$ and $\mathbf{V}_0$ and between $X'$ and $\mathbf{U}$.
	Therefore, $(\widetilde{\mathbf{X}} \setminus\{X'\}) \subset An(\widetilde{\mathbf{Y}})_{\widetilde{G}} \cup Ind(\widetilde{\mathbf{Y}})_{\widetilde{G}}$. 
	The remaining claims of the lemma is shown by lemma \ref{lemma_edge_del_modify_between_x_also_identifibility_preserve}.
	
	Next, we will show in the case that $\mathbf{X} \cap Ind(\mathbf{Y})_G \neq \phi$.
	Let $\mathbf{X}_1=\mathbf{X}\setminus Ind(\mathbf{Y})_G$ and let  $\mathbf{X}_2=\mathbf{X} \cap Ind(\mathbf{Y})_G$.
	Since $(\mathbf{Y}\indep \mathbf{X}_2)_G$, it holds that $(\mathbf{Y}\indep \mathbf{X}_2|\mathbf{X}_1)_G$.
	Therefore, by $do$-calculus R3, $P(\mathbf{Y}|do(\mathbf{X})) = P(\mathbf{Y}|do(\mathbf{X}_1))$.
	Hence, We can aplly $P(\mathbf{Y}|do(\mathbf{X}_1))$ to the above proof of the case that $\mathbf{X} \cap Ind(\mathbf{Y})_G = \phi$ 
	in $G$.	
	In consequence, $\widetilde{\mathbf{V}} \cup \widetilde{\mathbf{U}}$ is a semi-Markovian model
	and let  $\widetilde{G}$ be a DAG for  $\widetilde{\mathbf{V}} \cup \widetilde{\mathbf{U}}$,
	then $(\widetilde{\mathbf{X}}_1 \setminus \{X'\}) \subset An(\widetilde{\mathbf{Y}})_{\widetilde{G}}
	\cup Ind(\widetilde{\mathbf{Y}})_{\widetilde{G}}$.
	Moreover, $P(\widetilde{\mathbf{Y}}|do(\widetilde{\mathbf{X}}_1\setminus\{X'\}))$ is identifiable in $\widetilde{G}$.
	
	Note that, $(\widetilde{\mathbf{X}}_2\setminus \{X'\} \indep \widetilde{\mathbf{Y}})_{\widetilde{G}}$.
	In fact, because of (\ref{eq_indep_pa_tv_include_tilde_pav_x}),
	let $\mathbf{W}= Pa(\mathbf{Y})_G \cap (\mathbf{X}_2\setminus\{X'\})$, 
	then
	\begin{equation}
		Pa(\widetilde{\mathbf{Y}})_{\widetilde{G}} \cap 
		(\widetilde{\mathbf{X}}_2\setminus\{X'\})\subset \widetilde{Pa}(\mathbf{Y})_G \cap (\widetilde{\mathbf{X}}_2\setminus\{X'\}) = \widetilde{W} = \phi. \nonumber
	\end{equation}
	Therefore, by $do$-calculus R3 to delete $do(\widetilde{\mathbf{X}}_2 \setminus \{X'\})$
	of $do(\widetilde{\mathbf{X}}\setminus \{X'\})$, 
	we obtain $P(\widetilde{\mathbf{Y}}|do(\widetilde{\mathbf{X}}\setminus \{X'\})) =
	P(\widetilde{\mathbf{Y}}|do(\widetilde{\mathbf{X}}_1\setminus \{X'\}))$ in $\widetilde{G}$.
	Recall that  $P(\widetilde{\mathbf{Y}}|do(\widetilde{\mathbf{X}}_1\setminus\{X'\}))$ is identifiable in $\widetilde{G}$. 
	Therefore, $P(\widetilde{\mathbf{Y}}|do(\widetilde{\mathbf{X}}\setminus\{X'\}))$ is identifiable in $\widetilde{G}$.
	Finally, since
	$(\widetilde{\mathbf{X}}_1 \setminus \{X'\}) \subset An(\widetilde{\mathbf{Y}})_{\widetilde{G}}
	\cup Ind(\widetilde{\mathbf{Y}})_{\widetilde{G}}$ and 
	since
	$(\widetilde{\mathbf{X}}_2 \setminus \{X'\}) \subset Ind(\widetilde{\mathbf{Y}})_{\widetilde{G}}$,
	we obtain
	$\widetilde{\mathbf{X}}\setminus\{X'\} 
	= (\widetilde{\mathbf{X}}_1\cup \widetilde{\mathbf{X}}_2) \setminus \{X'\}
	\subset An(\widetilde{\mathbf{Y}})_{\widetilde{G}}
	\cup Ind(\widetilde{\mathbf{Y}})_{\widetilde{G}}$.

	This completes the proof of the lemma.
\end{proof}

\begin{lemma}\label{lemma_condi_minus_modify_also_identifibility_preserve}
	Let G be a DAG for V and U.
	Assume that for disjoint sets $\mathbf{Y},\mathbf{X},\mathbf{Z} \subset \mathbf{V}$,
	$\mathbf{X} \subset An(\mathbf{Y})_G \cup Ind(\mathbf{Y})_{G}$.
	Let $\mathbf{z}\in\mathcal{X}_{\mathbf{Z}}$ be fixed.
	For $A\in \mathbf{V}\cup\mathbf{U}$, let
	\begin{equation}
		\widetilde{A} = 
		\begin{cases}
			A - E[A|\mathbf{Z}=\mathbf{z}] & \text{if $A \notin \mathbf{Z}$} \nonumber\\
			A                              & \text{if $A \in \mathbf{Z}$}. \nonumber
		\end{cases} \nonumber
	\end{equation}
	For $\mathbf{A} \subset \mathbf{V}\cup\mathbf{U}$, let $\widetilde{\mathbf{A}} = \{\widetilde{A}|A\in\mathbf{A}\}$.
	Then, $\widetilde{\mathbf{V}}\cup\widetilde{\mathbf{U}}$ is a semi-Markovian model.
	Let $\widetilde{G}$ be a DAG for $\widetilde{\mathbf{V}}\cup\widetilde{\mathbf{U}}$,
	then $\widetilde{\mathbf{X}} \subset An(\widetilde{\mathbf{Y}})_{\widetilde{G}} \cup Ind(\widetilde{\mathbf{Y}})_{\widetilde{G}}$.
	Moreover, if $P(\mathbf{Y}|do(\mathbf{X}), \mathbf{Z})$ is identifiable in $G$, then
	$P(\widetilde{\mathbf{Y}}|do(\widetilde{\mathbf{X}}))$ is identifiable in $\widetilde{G}$.
\end{lemma}

\begin{proof}[proof of Lemma \ref{lemma_condi_minus_modify_also_identifibility_preserve}]
	Initially, we will show the lemma in the case that $\mathbf{X} \cap Ind(\mathbf{Y})_G = \phi$.
	Let $\mathbf{V}_0 = \mathbf{V}\setminus\mathbf{Z}$. 
	Note that, for any $V_0 \in \mathbf{V}_0$,
	\begin{equation}
		\widetilde{V}_0 \indep \mathbf{Z}. \label{eq_indep_wv_z} 
	\end{equation}
	in fact, for any variables $\xi$ and $\eta$, it holds that
	\begin{equation}
		\xi \indep \eta \quad \Longleftrightarrow \quad E[\xi|\eta] = E[\xi], \nonumber
	\end{equation}
	and for $\widetilde{V}_0$ and $\mathbf{Z}$,
	\begin{eqnarray}
		E[\widetilde{V}_0|\mathbf{Z}] &=& E[V_0-E[V_0|\mathbf{Z}]|\mathbf{Z}] \nonumber\\
		&=& E[V_0|\mathbf{Z}] - E[V_0|\mathbf{Z}] \nonumber \\
		&=& 0  = E[\widetilde{V}_0]. \nonumber
	\end{eqnarray}
	Thus,
	\begin{equation}
		P(\widetilde{V}_0, \mathbf{Z}) = P(\widetilde{V}_0)P(\mathbf{Z}). \nonumber
	\end{equation}
	Similarly, for $U\in \mathbf{U}$,
	\begin{equation}
		P(\widetilde{U}, \mathbf{Z}) = P(\widetilde{U})P(\mathbf{Z}). \nonumber
	\end{equation}
	By the way, for $v \in \mathcal{X}_{V_0}$ and $u \in \mathcal{X}_{U}$, let 
	\[
	\widetilde{v} = v - E[V_0|\mathbf{Z}], \quad \widetilde{u} = u - E[U|\mathbf{Z}],
	\]
	then according to the definition of $\widetilde{V}_0$ and $\widetilde{U}$, 
	\begin{equation}
		P(\widetilde{V}_0=\widetilde{v}, \widetilde{U}=\widetilde{u}) = P(V=v, U=u|\mathbf{Z}=\mathbf{z}). \nonumber
	\end{equation}
	Therefore,
	\begin{eqnarray}
		P(\widetilde{V}_0=\widetilde{v}, \widetilde{U}=\widetilde{u}, \widetilde{\mathbf{Z}}=\widetilde{\mathbf{z}}) &=&
		P(\widetilde{V}_0=\widetilde{v}, \widetilde{U}=\widetilde{u}|\widetilde{\mathbf{Z}}=\widetilde{\mathbf{z}}) P(\widetilde{\mathbf{Z}}=\widetilde{\mathbf{z}}) \nonumber\\
		&=& P(\widetilde{V}_0=\widetilde{v}, \widetilde{U}=\widetilde{u}) P(\widetilde{\mathbf{Z}}=\widetilde{\mathbf{z}}) \nonumber\\
		&=&  P(V_0=v, U=u|\mathbf{Z}=\mathbf{z}) P(\mathbf{Z}=\mathbf{z}) \nonumber\\
		&=&  P(V_0=v, U=u,\mathbf{Z}=\mathbf{z}). \nonumber
	\end{eqnarray}
	Similarly, for $\mathbf{V}' \subset \mathbf{V}_0$, $\mathbf{U}' \subset \mathbf{U}_G$, 
	$\mathbf{v} \in \mathcal{X}_{\mathbf{V}'}$ and $\mathbf{u} \in \mathcal{X}_{\mathbf{U}'}$, let
	\[
	\widetilde{\mathbf{v}} = \mathbf{v} - E[\mathbf{V}'|\mathbf{Z}], \quad \widetilde{\mathbf{u}} = \mathbf{u} - E[\mathbf{U}'|\mathbf{Z}],
	\]	
	then,
	\begin{equation}
		P(\mathbf{V}'=\mathbf{v}, \mathbf{U}'=\mathbf{u}, \mathbf{Z}=\mathbf{z}) = 
		P(\widetilde{\mathbf{V}}'=\widetilde{\mathbf{v}}, \widetilde{\mathbf{U}}'=\widetilde{\mathbf{u}}, \mathbf{Z}=\mathbf{z}). \nonumber
	\end{equation}
	Therefore, 
	\begin{eqnarray}
		\lefteqn{P(V=v|Pa(V)_G=pa_V, UPa(V)_G=u_V)} \quad \nonumber\\
		&=& P(\widetilde{V}=\widetilde{v}|\widetilde{Pa}(V)_{\widetilde{G}}=\widetilde{pa_V}, 
		\widetilde{UPa}(V)_{\widetilde{G}}=\widetilde{u_V}), \label{eq_preserve_parents}
	\end{eqnarray}
	where $pa_V$ and $u_V$ are the values of $Pa(V)_G$ and $UPa(V)_G$. \\
	Therefore, 
	\begin{equation}
		Pa(\widetilde{V})_{\widetilde{G}} \subset \widetilde{Pa}(V)_G \label{eq_indep_pa_tv_include_tilde_pav}
	\end{equation}
	\begin{equation}
		UPa(\widetilde{V})_{\widetilde{G}} \subset \widetilde{UPa}(V)_G \label{eq_indep_upa_tv_include_tilde_upav}
	\end{equation}
	Now, for $V \in \mathbf{V}$, $W\in Pa(V)_G$ and $U\in UPa(V)_G$, 
	the following equations (\ref{at_proof_lemma_condi_minus_in_case_v_in_z_w_in_z})-(\ref{at_proof_lemma_condi_minus_in_case_v_nin_z_then_u_in_pav})
	hold.
	\begin{equation}
		V \in \mathbf{Z} \quad and \quad W \in \mathbf{Z}  \quad \Longrightarrow \quad  \widetilde{W} \in Pa(\widetilde{V})_{\widetilde{G}}
		\label{at_proof_lemma_condi_minus_in_case_v_in_z_w_in_z}
	\end{equation}
	\begin{equation}
		V \in \mathbf{Z} \quad and \quad W \notin \mathbf{Z} \quad \Longrightarrow \quad  \widetilde{V} \indep \widetilde{W} \
		\label{at_proof_lemma_condi_minus_in_case_v_in_z_w_nin_z}
	\end{equation}
	\begin{equation}
		V \notin \mathbf{Z} \quad and \quad W \in \mathbf{Z} \quad \Longrightarrow \quad \widetilde{V} \indep \widetilde{W}
		\label{at_proof_lemma_condi_minus_in_case_v_nin_z_w_in_z}
	\end{equation}
	\begin{equation}
		V \notin \mathbf{Z} \quad and \quad W \notin \mathbf{Z} \quad \Longrightarrow \quad \widetilde{W} \in Pa(\widetilde{V})_{\widetilde{G}}
		\label{at_proof_lemma_condi_minus_in_case_v_nin_z_w_nin_z}
	\end{equation}
	\begin{equation}
		V \in \mathbf{Z} \quad \Longrightarrow \quad  \widetilde{V} \indep \widetilde{U} 
		\label{at_proof_lemma_condi_minus_in_case_v_in_z_the_idp_u}
	\end{equation}
	\begin{equation}
		V \notin \mathbf{Z} \quad \Longrightarrow \quad \widetilde{U} \in UPa(\widetilde{V})_{\widetilde{G}}.
		\label{at_proof_lemma_condi_minus_in_case_v_nin_z_then_u_in_pav}
	\end{equation}
	(\ref{at_proof_lemma_condi_minus_in_case_v_in_z_w_in_z}) is clear from $\widetilde{V}=V$, $\widetilde{W}=W$.
	(\ref{at_proof_lemma_condi_minus_in_case_v_in_z_w_nin_z}), (\ref{at_proof_lemma_condi_minus_in_case_v_nin_z_w_in_z})
	and (\ref{at_proof_lemma_condi_minus_in_case_v_in_z_the_idp_u}) hold from (\ref{eq_indep_wv_z}).
	For (\ref{at_proof_lemma_condi_minus_in_case_v_nin_z_w_nin_z}), assume
	that $V \notin \mathbf{Z}$ and $W \notin \mathbf{Z}$ and $\widetilde{W} \notin Pa(\widetilde{V})_G$.
	From (\ref{eq_indep_pa_tv_include_tilde_pav}), $\widetilde{W} \in \widetilde{Pa}(V)_G \setminus Pa(\widetilde{V})_{\widetilde{G}}$.
	Thus, at the right hand side (\ref{eq_preserve_parents}),
	\begin{equation}
		P(\widetilde{V}|\widetilde{Pa}(V)_{\widetilde{G}}, \widetilde{UPa}(V)_{\widetilde{G}}) 
		= P(\widetilde{V}|\widetilde{Pa}(V)_{\widetilde{G}}\setminus \widetilde{W}, \widetilde{UPa}(V)_{\widetilde{G}}). \nonumber
	\end{equation}	
	Therefore,
	\begin{equation}
		\widetilde{V} \indep \widetilde{W}. \label{eq_indep_tv_tw} 
	\end{equation}
	Now, according to (\ref{eq_indep_wv_z}), $\widetilde{V} \indep \mathbf{Z}$ and $\widetilde{W} \indep \mathbf{Z}$,
	thus $\widetilde{V} \indep E[W|\mathbf{Z}]$ and $\widetilde{W} \indep E[V|\mathbf{Z}]$.

	Now, by (\ref{eq_indep_wv_z}), it holds that $\widetilde{V} \indep \mathbf{Z}$ and
	$\widetilde{W} \indep \mathbf{Z}$.
	Thus, 
	\[
	\widetilde{V} \indep \widetilde{W} + E[W|\mathbf{Z}] = W,
	\]
	and
	\[
	\widetilde{W} \indep \widetilde{V} + E[V|\mathbf{Z}] = V.
	\]
	Therefore, since
	\begin{equation*}
		E[\widetilde{V}+E[V|\mathbf{Z}]|\widetilde{W},\mathbf{Z}] = E[\widetilde{V}|\widetilde{W},\mathbf{Z}] - E[E[V|\mathbf{Z}]|\widetilde{W},\mathbf{Z}] = E[V|\mathbf{Z}],
	\end{equation*}
	it holds that
	\begin{eqnarray*}
		E[V|W] &=&  E[\widetilde{V}+E[V|\mathbf{Z}]|\widetilde{W}+E[W|\mathbf{Z}], \widetilde{W}, \mathbf{Z}]\\
		&=& E[E[V|\mathbf{Z}]|\widetilde{W}+E[W|\mathbf{Z}], \widetilde{W}, \mathbf{Z}] \\
		&=& E[V|\mathbf{Z}]
	\end{eqnarray*}
	Since any variables $A$, $\mathbf{B}$ and $C$, if $E[A|\mathbf{B}] = E[A|C]$, then
	there exists a function $f:\mathcal{X}_{\mathbf{B}} \rightarrow \mathcal{X}_C$ such that 
	\begin{equation*}
		P(\mathbf{B}=\mathbf{b}) = P(C=f(\mathbf{b})),
	\end{equation*}
	there exists a function $f:\mathcal{X}_{\mathbf{Z}} \rightarrow \mathcal{X}_{W}$ such that $W=f(\mathbf{Z})$.
	Then, $W \in \mathbf{Z}$ because we assume that $V \neq f(\mathbf{W})$ for any 
	function $f$, $V \in \mathbf{V}$ and $\mathbf{W}\subset \mathbf{V} \setminus V$.
	Therefore, it contradicts that $W \notin \mathbf{Z}$. 
	Therefore, (\ref{at_proof_lemma_condi_minus_in_case_v_nin_z_w_nin_z}) holds.
	For (\ref{at_proof_lemma_condi_minus_in_case_v_nin_z_then_u_in_pav}), 
	by replacing $Pa(V)_G$ with $UPa(V)_G$ and 
	replacing $Pa(\widetilde{V})_{\widetilde{G}}$ with $UPa(\widetilde{V})_{\widetilde{G}}$
	at the above proof of (\ref{at_proof_lemma_condi_minus_in_case_v_nin_z_w_nin_z}) 
	because always $U \notin \mathbf{Z}$.
	
	Now, We consider a DAG $G'$ obtained from G by replacing each $A\in \mathbf{V}\cup\mathbf{U}$ with $\widetilde{A}$.
	Since $P(\mathbf{Y}|do(\mathbf{X}), \mathbf{Z})$ is identifiable in $G$, $P(\widetilde{\mathbf{Y}}|do(\widetilde{\mathbf{X}}), \mathbf{Z})$ is 
	identifiable in $G'$. Because of 
	(\ref{at_proof_lemma_condi_minus_in_case_v_in_z_w_in_z})-(\ref{at_proof_lemma_condi_minus_in_case_v_nin_z_then_u_in_pav}),
	$\widetilde{G}$ is just a graph obtained by deleting the arrows in $G'$
	between $\mathbf{Z}$ and  $\mathbf{V}\setminus\mathbf{Z}$
	and between $\mathbf{Z}$ and $\mathbf{U}$.
	Therefore, $\widetilde{\mathbf{X}} \subset An(\widetilde{\mathbf{Y}})_{\widetilde{G}} \cup Ind(\widetilde{\mathbf{Y}})_{\widetilde{G}}$.
	The remaining claims of the lemma is shown by lemma \ref{lemma_edge_del_modify_between_z_also_identifibility_preserve}.
	
	Next, we will show in the case that $\mathbf{X} \cap Ind(\mathbf{Y})_G \neq \phi$. 
	Let
	$\mathbf{X}_1=\mathbf{X}\setminus Ind(\mathbf{Y})_G$ and let $\mathbf{X}_2=\mathbf{X} \cap Ind(\mathbf{Y})_G$.
	Since $\mathbf{Y}\indep \mathbf{X}_2$, it holds that $\mathbf{Y}\indep \mathbf{X}_2|\mathbf{X}_1$.
	Therefore, by $do$-calculs R3, $P(\mathbf{Y}|do(\mathbf{X})) = P(\mathbf{Y}|do(\mathbf{X}_1))$.
	Hence, We can apply
	$P(\mathbf{Y}|do(\mathbf{X}_1))$ to the above proof of the case that $\mathbf{X} \cap Ind(\mathbf{Y})_G = \phi$.
	In consequence, $\widetilde{\mathbf{V}} \cup \widetilde{\mathbf{U}}$ is a semi-Markovian model and 
	let $\widetilde{G}$ be a DAG for $\widetilde{\mathbf{V}} \cup \widetilde{\mathbf{U}}$, 
	then $\widetilde{\mathbf{X}}_1 \subset An(\widetilde{\mathbf{Y}})_{\widetilde{G}}
	\cup Ind(\widetilde{\mathbf{Y}})_{\widetilde{G}}$.
	Moreover, $P(\widetilde{\mathbf{Y}}|do(\widetilde{\mathbf{X}}_1))$ is identifiable in $\widetilde{G}$.
	
	Note that, $\widetilde{\mathbf{X}}_2 \indep \widetilde{\mathbf{Y}}$.
	In fact, because of (\ref{eq_indep_pa_tv_include_tilde_pav_x}),
	let $\mathbf{W}= Pa(\mathbf{Y})_G \cap \mathbf{X}_2$, then
	\begin{equation}
		Pa(\widetilde{\mathbf{Y}})_{\widetilde{G}} \cap 
		\widetilde{\mathbf{X}}_2 \subset \widetilde{Pa}(\mathbf{Y})_G \cap \widetilde{\mathbf{X}}_2 = \widetilde{W} = \phi \nonumber
	\end{equation}
	Therefore, by $do$-calculs R3 to delete $do(\widetilde{\mathbf{X}}_2)$ of $do(\widetilde{\mathbf{X}})$,
	$P(\widetilde{\mathbf{Y}}|do(\widetilde{\mathbf{X}})) =P(\widetilde{\mathbf{Y}}|do(\widetilde{\mathbf{X}}_1))$ in $\widetilde{G}$.
	Recall that $P(\widetilde{\mathbf{Y}}|do(\widetilde{\mathbf{X}}_1))$ is identifiable in $\widetilde{G}$.
	Therefore, $P(\widetilde{\mathbf{Y}}|do(\widetilde{\mathbf{X}}))$ is identifiable in $\widetilde{G}$.
	Finally, since $\widetilde{\mathbf{X}}_1 \subset An(\widetilde{\mathbf{Y}})_{\widetilde{G}}
	\cup Ind(\widetilde{\mathbf{Y}})_{\widetilde{G}}$ and 
	$\widetilde{\mathbf{X}}_2 \subset Ind(\widetilde{\mathbf{Y}})_{\widetilde{G}}$, 
	we obtain 
	$\widetilde{\mathbf{X}}	= \widetilde{\mathbf{X}}_1\cup \widetilde{\mathbf{X}}_2
	\subset An(\widetilde{\mathbf{Y}})_{\widetilde{G}}
	\cup Ind(\widetilde{\mathbf{Y}})_{\widetilde{G}}$.
	
	This completes the proof of the lemma.
\end{proof}

\begin{proof}[proof of Theorem \ref{thorem_multi_vals}]
	Let $\widetilde{\mathbf{V}}^{(k)}=\{\widetilde{Y}^{(k)}\}\cup\widetilde{\mathbf{X}}^{(k)}\cup\mathbf{Z}$ and let
	$\widetilde{\mathbf{U}}^{(k)}=\{\widetilde{U}^{(k)}|U\in\mathbf{U}\}$.
	By lemma \ref{lemma_condi_minus_modify_also_identifibility_preserve}, 
	$\widetilde{\mathbf{V}}^{(1)}\cup\widetilde{\mathbf{U}}^{(1)}$ is a semi-Markovian model. 
	Let $\widetilde{G}^{(1)}$ be a DAG for $\widetilde{\mathbf{V}}^{(1)}\cup\widetilde{\mathbf{U}}^{(1)}$,
	then $\widetilde{\mathbf{X}}^{(1)} \subset An(\widetilde{Y}^{(1)})_{\widetilde{G}^{(1)}} \cup Ind(\widetilde{Y}^{(1)})_{\widetilde{G}^{(1)}}$.
	In addition, from the identifiability of  $P(Y|do(\mathbf{X}),\mathbf{Z})$ in $G$,
	$P(\widetilde{Y}^{(1)}|do(\widetilde{\mathbf{X}}^{(1)}))$ is identifiable in $\widetilde{G}^{(1)}$.
	Moreover, it holds that
	\begin{equation}
		\widetilde{Y}^{(1)}, \widetilde{\mathbf{X}}^{(1)} \indep \mathbf{Z}. \nonumber
	\end{equation}
	In fact, for any variables $\xi$ and $\eta$, it holds that
	\begin{equation}
		\xi \indep \eta \quad \Longleftrightarrow \quad E[\xi|\eta] = E[\xi], \nonumber
	\end{equation}
	and for $\widetilde{Y}^{(1)}$ and $\mathbf{Z}$,
	\begin{equation}
		E[\widetilde{Y}^{(1)}|\mathbf{Z}] = E[Y - E[Y|\mathbf{Z}]|\mathbf{Z}] = E[Y|\mathbf{Z}] - E[Y|\mathbf{Z}] = 0 = E[\widetilde{Y}^{(1)}]. \nonumber 
	\end{equation}
	For $\widetilde{X}_i^{(1)}$ and $\mathbf{Z}$ $(1 \le i \le n)$, it holds that
	\begin{equation}
		E[\widetilde{X}_i^{(1)}|\mathbf{Z}] = E[X_i - E[X_i|\mathbf{Z}]|\mathbf{Z}] = E[X_i|\mathbf{Z}] - E[X_i|\mathbf{Z}] = 0 = E[\widetilde{X}_i^{(1)}]. \nonumber
	\end{equation}
	Additionally, 
	\begin{eqnarray}
		E[\widetilde{Y}^{(2)}| \widetilde{X}_1^{(1)}]
		&=& E[\widetilde{Y}^{(1)} - E[\widetilde{Y}^{(1)}|\widetilde{X}_1^{(1)}]|\widetilde{X}_1^{(1)}] \nonumber\\
		&=& E[\widetilde{Y}^{(1)}|\widetilde{X}_1^{(1)}] -E[\widetilde{Y}^{(1)}|\widetilde{X}_1^{(1)}] \nonumber \\
		&=& 0 = E[\widetilde{Y}^{(1)}] \nonumber\\
		E[\widetilde{X}_2^{(2)}|\widetilde{X}_1^{(1)}] &=& E[\widetilde{X}_2^{(1)} - E[\widetilde{X}_2^{(1)}|\widetilde{X}_1^{(1)}]|\widetilde{X}_1^{(1)}] \nonumber\\
		&=& E[\widetilde{X}_2^{(1)}|\widetilde{X}_1^{(1)}] - E[\widetilde{X}_2^{(1)}|\widetilde{X}_1^{(1)}] \nonumber\\
		&=& 0 = E[\widetilde{X}_1^{(1)}]. \nonumber
	\end{eqnarray}
	Summarizing above results,
	\begin{eqnarray}
		\widetilde{Y}^{(1)},\widetilde{\mathbf{X}}^{(1)} &\indep& \mathbf{Z} \nonumber\\
		\widetilde{Y}^{(2)}, \widetilde{X}_2^{(2)} &\indep& \widetilde{X}_1^{(1)}. \nonumber
	\end{eqnarray}
	In addition, it also holds that
	\begin{eqnarray}
		\widetilde{Y}^{(2)} , \widetilde{X}_2^{(2)} &\indep& \mathbf{Z}. \nonumber
	\end{eqnarray}		
	In fact,
	\begin{eqnarray}
		E[\widetilde{Y}^{(2)}|\mathbf{Z}] &=& E[\widetilde{Y}^{(1)} - E[\widetilde{Y}^{(1)}|\widetilde{X}_1^{(1)}]|\mathbf{Z}] \nonumber\\
		&=&  E[\widetilde{Y}^{(1)}|\mathbf{Z}] - E[E[\widetilde{Y}^{(1)}|\widetilde{X}_1^{(1)} ] \mathbf{Z}] \nonumber \\
		&=&  E[\widetilde{Y}^{(1)}] -E[E[\widetilde{Y}^{(1)}|\widetilde{X}_1^{(1)}]] \qquad (\because \widetilde{Y}^{(1)}, \widetilde{X}_1^{(1)} \indep \mathbf{Z}) \nonumber\\
		&=&  E[\widetilde{Y}^{(1)}] -E[\widetilde{Y}^{(1)}] \nonumber\\
		&=& 0  \nonumber \\
		&=&	E[\widetilde{Y}^{(2)}]  \nonumber
	\end{eqnarray}
	\begin{eqnarray}
		E[\widetilde{X}_2^{(2)}|\mathbf{Z}] &=& E[\widetilde{X}_2^{(1)} - E[\widetilde{X}_2^{(1)}|\widetilde{X}_1^{(1)}]|\mathbf{Z}] \nonumber\\
		&=&  E[\widetilde{X}_2^{(1)}|\mathbf{Z}] - E[E[\widetilde{X}_2^{(1)}|\widetilde{X}_1^{(1)}]| \mathbf{Z}] \nonumber\\
		&=&  E[\widetilde{X}_2^{(1)}] - E[E[\widetilde{X}_2^{(1)}|\widetilde{X}_1^{(1)}]]  \qquad(\because \widetilde{X}_2^{(1)}, \widetilde{X}_1^{(1)} \indep \mathbf{Z})\nonumber \\
		&=&  E[\widetilde{X}_2^{(1)}] - E[\widetilde{X}_2^{(1)}] \nonumber\\
		&=&  0  \nonumber\\
		&=&  E[\widetilde{X}_2^{(2)}].   \nonumber
	\end{eqnarray}
	Since lemma \ref{lemma_one_do_minus_modify_also_identifibility_preserve}, 
	$\widetilde{\mathbf{V}}^{(2)}\cup\widetilde{\mathbf{U}}^{(2)}$ is a semi-Markovian model.
	Let $\widetilde{G}^{(2)}$ be a DAG for $\widetilde{\mathbf{V}}^{(2)}\cup\widetilde{\mathbf{U}}^{(2)}$,
	then $\widetilde{\mathbf{X}}^{(2)} \subset An(\widetilde{Y}^{(2)})_{\widetilde{G}^{(2)}} \cup Ind(\widetilde{Y}^{(2)})_{\widetilde{G}^{(2)}}$.
	Because $P(Y|do(\mathbf{X}^{(1)}),\mathbf{Z}^{(1)})$ is identifiable in $\widetilde{G}^{(1)}$,
	$P(\widetilde{Y}^{(2)}|do(\widetilde{\mathbf{X}}^{(2)}))$ is identifiable in $\widetilde{G}^{(2)}$.
	
	Now, assume the following as inductive assumptions for $k \ge 1$.
	\begin{enumerate}
		\item $\widetilde{Y}^{(k)} \indep \{\widetilde{X}_1^{(1)}, \widetilde{X}_2^{(2)}, \dots,\widetilde{X}_{k-1}^{(k-1)}\}$ \label{cond-1}
		\item $\widetilde{X}_i^{(k)} \indep \{\widetilde{X}_1^{(1)}, \widetilde{X}_2^{(2)}, \dots,\widetilde{X}_{k-1}^{(k-1)}\}$  \quad ($\forall i \ge k$) \label{cond-2}
		\item $\widetilde{Y}^{(k)}, \widetilde{\mathbf{X}}^{(k)} \indep \mathbf{Z}$ \label{cond-3}
		\item $\widetilde{\mathbf{V}}^{(k)}\cup\widetilde{\mathbf{U}}^{(k)}$ is a semi-Markovian model. let $\widetilde{G}^{(k)}$ be
		a DAG for 
		$\widetilde{\mathbf{V}}^{(k)}\cup\widetilde{\mathbf{U}}^{(k)}$, then
		$\widetilde{\mathbf{X}}^{(k)} \subset An(\widetilde{Y}^{(k)})_{\widetilde{G}^{(k)}} \cup Ind(\widetilde{Y}^{(k)})_{\widetilde{G}^{(k)}}$ 
		and $P(\widetilde{Y}^{(k)}|do(\widetilde{\mathbf{X}}^{(k)}))$ is identifiable in $\widetilde{G}^{(k)}$. \label{cond-4}
	\end{enumerate}
	Firstly,
	\begin{eqnarray}
		E[\widetilde{Y}^{(k+1)}| \widetilde{X}_k^{(k)}] &=& E[\widetilde{Y}^{(k)} - E[\widetilde{Y}^{(k)}|\widetilde{X}_k^{(k)}]|\widetilde{X}_k^{(k)}] \nonumber\\
		&=& E[\widetilde{Y}^{(k)}|\widetilde{X}_k^{(k)}] -E[\widetilde{Y}^{(k)}|\widetilde{X}_k^{(k)}] \nonumber\\
		&=& 0 = E[\widetilde{Y}^{(k+1)}]. \nonumber
	\end{eqnarray}
	In addition, for $1 \le i < k$, it also holds that
	\begin{equation}
		\{x \in \mathcal{X}_{\widetilde{X}_k^{(k)}} | \mathcal{X}_{\widetilde{X}_i^{(i)}} = \widetilde{x}_i^{(i)} \} = \mathcal{X}_{\widetilde{X}_k^{(k)}}. \nonumber
	\end{equation}
	Therefore, for fixed $\widetilde{x}_i^{(i)} \in \mathcal{X}_{\widetilde{X}_i^{(i)}}$, 
	\begin{eqnarray}
		E[\widetilde{Y}^{(k+1)}|\widetilde{X}_i^{(i)}=\widetilde{x}_i^{(i)}] &=& E[\widetilde{Y}^{(k)} - E[\widetilde{Y}^{(k)}|\widetilde{X}_k^{(k)}]|\widetilde{X}_i^{(i)}=\widetilde{x}_i^{(i)}] \nonumber\\
		&=&  E[\widetilde{Y}^{(k)}|\widetilde{X}_i^{(i)}=\widetilde{x}_i^{(i)}] - E[E[\widetilde{Y}^{(k)}|\widetilde{X}_k^{(k)}]|\widetilde{X}_i^{(i)}=\widetilde{x}_i^{(i)}]\nonumber \\
		&=&  E[\widetilde{Y}^{(k)}|\widetilde{X}_i^{(i)}=\widetilde{x}_i^{(i)}] -   E[\widetilde{Y}^{(k)}|\widetilde{X}_i^{(i)}=\widetilde{x}_i^{(i)}] \nonumber\\
		&=& 0  \nonumber \\
		&=&	E[\widetilde{Y}^{(k+1)}].  \nonumber
	\end{eqnarray}
	Hence, the assumption \ref{cond-1} holds when $k+1$.
	
	Next, it holds that
	\begin{eqnarray}
		E[\widetilde{X}_{k+1}^{(k+1)}|\widetilde{X}_{k}^{(k)}] &=& E[\widetilde{X}_{k+1}^{(k)} - E[\widetilde{X}_{k+1}^{(k)}|\widetilde{X}_{k}^{(k)}]|\widetilde{X}_k^{(k)}] \nonumber\\
		&=& E[\widetilde{X}_{k+1}^{(k)}|\widetilde{X}_{k}^{(k)}] - E[\widetilde{X}_{k+1}^{(k)}|\widetilde{X}_{k}^{(k)}] \nonumber\\
		&=& 0 = E[\widetilde{X}_{k+1}^{(k+1)}]. \nonumber
	\end{eqnarray}
	In addition, for $j \ge k+1$, $i < k$ it also holds that
	\begin{eqnarray}
		E[\widetilde{X}_{j}^{(k+1)}|\widetilde{X}_i^{(i)}=\widetilde{x}_i^{(i)}] &=& 
		E[\widetilde{X}_{j}^{(k)} - E[\widetilde{X}_{j}^{(k)}|\widetilde{X}_k^{(k)}]|\widetilde{X}_i^{(i)}=\widetilde{x}_i^{(i)}] \nonumber\\
		&=&  E[\widetilde{X}_{j}^{(k)}|\widetilde{X}_i^{(i)}=\widetilde{x}_i^{(i)}] - 
		E[E[\widetilde{X}_{j}^{(k)}|\widetilde{X}_k^{(k)}]|\widetilde{X}_i^{(i)}=\widetilde{x}_i^{(i)}]\nonumber \\
		&=&  E[\widetilde{X}_{j}^{(k)}|\widetilde{X}_i^{(i)}=\widetilde{x}_i^{(i)}] -   E[\widetilde{X}_{j}^{(k)}|\widetilde{X}_i^{(i)}=\widetilde{x}_i^{(i)}] \nonumber\\
		&=& 0  \nonumber \\
		&=&	E[\widetilde{X}_{j}^{(k+1)}].  \nonumber
	\end{eqnarray}
	Hence, the assumption \ref{cond-2} holds when $k+1$.

	As for the assumption \ref{cond-3}, it holds that
	\begin{eqnarray}
		E[\widetilde{Y}^{(k+1)}|\mathbf{Z}] &=& E[\widetilde{Y}^{(k)} - E[\widetilde{Y}^{(k)}|\widetilde{X}_k^{(k)}]|\mathbf{Z}] \nonumber\\
		&=&  E[\widetilde{Y}^{(k)}|\mathbf{Z}] - E[E[\widetilde{Y}^{(k)}|\widetilde{X}_k^{(k)}]|\mathbf{Z}]\nonumber \\
		&=&  E[\widetilde{Y}^{(k)}|\mathbf{Z}] -   E[\widetilde{Y}^{(k)}|\widetilde{X}_k^{(k)}] \nonumber\\
		&=&  E[\widetilde{Y}^{(k)}] -   E[\widetilde{Y}^{(k)}] \nonumber\\
		&=& 0  \nonumber \\
		&=&	E[\widetilde{Y}^{(k+1)}].  \nonumber
	\end{eqnarray}
	In addition, for $i \ge k+1$, it holds that
	\begin{eqnarray}
		E[\widetilde{X}_{i}^{(k+1)}|\mathbf{Z}] &=& E[\widetilde{X}_{i}^{(k)} - E[E[\widetilde{X}_{i}^{(k)}|\widetilde{X}_k^{(k)}]|\mathbf{Z}] \nonumber\\
		&=&  E[\widetilde{X}_{i}^{(k)}|\mathbf{Z}] - E[E[\widetilde{X}_{i}^{(k)}|\widetilde{X}_k^{(k)}]|\mathbf{Z}]\nonumber \\
		&=&  E[\widetilde{X}_{i}^{(k)}] -   E[E[\widetilde{X}_{i}^{(k)}|\widetilde{X}_k^{(k)}]] \quad (\because \widetilde{\mathbf{X}}^{(k)} \indep \mathbf{Z})\nonumber\\
		&=&  E[\widetilde{X}_{i}^{(k)}] -   E[\widetilde{X}_{i}^{(k)}] \nonumber\\
		&=& 0  \nonumber \\
		&=&	E[\widetilde{X}_{i}^{(k+1)}]  \nonumber
	\end{eqnarray}
	Hence, the assumption \ref{cond-3} holds when $k+1$.
	
	Finally, we show that the assumption \ref{cond-4} holds when $k+1$.
	Note that, $\widetilde{X}_{1}^{(k+1)}=\widetilde{X}_{1}^{(1)}, \widetilde{X}_{2}^{(k+2)}=\widetilde{X}_{2}^{(2)}, \ldots, \widetilde{X}_{k}^{(k+1)}=\widetilde{X}_{k}^{(k)}$.
	Since the assumption \ref{cond-1}, we can apply $do$-calculus R3 to
	$\widetilde{X}_1^{(k)}, \widetilde{X}_2^{(k)}, \dots,\widetilde{X}_{k-1}^{(k)}$ of $P(\widetilde{Y}^{(k)}|do(\widetilde{\mathbf{X}}^{(k)}))$, 
	and we obtain
	\begin{eqnarray}
		P(\widetilde{Y}^{(k)}|do(\widetilde{\mathbf{X}}^{(k)})) &=& P(\widetilde{Y}^{(k)}|do(\widetilde{X}_{1}^{(k)}), do(\widetilde{X}_{2}^{(k)}), \ldots,  do(\widetilde{X}_{n}^{(k)})) \nonumber\\
		&=&  P(\widetilde{Y}^{(k)}|do(\widetilde{X}_{k}^{(k)}), do(\widetilde{X}_{k+1}^{(k)}), \ldots, do(\widetilde{X}_{n}^{(k)})) \nonumber
	\end{eqnarray}
	Therefore, from the assumption \ref{cond-4} , $P(\widetilde{Y}^{(k)}|do(\widetilde{X}_{k}^{(k)}), do(\widetilde{X}_{k+1}^{(k)}), \ldots, do(\widetilde{X}_{n}^{(k)}))$
	is identifiable in $\widetilde{G}^{(k)}$.
	Moreover, in a setting that $\widetilde{\mathbf{V}}=\widetilde{\mathbf{V}}^{(k)}$ 
	$\widetilde{\mathbf{U}}=\widetilde{\mathbf{U}}^{(k)}$ ,
	$\mathbf{X}=\{\widetilde{X}_{k}^{(k)}, \widetilde{X}_{k+1}^{(k)}, \ldots, \widetilde{X}_{n}^{(k)}\}$ and
	$X'=\widetilde{X}_{k}^{(k)}$, 
	by the lemma \ref{lemma_one_do_minus_modify_also_identifibility_preserve}, 
	$\widetilde{\mathbf{V}}^{(k+1)}\cup\widetilde{\mathbf{U}}^{(k+1)}$ is a semi-Markovian model.
	Let  $\widetilde{G}^{(k+1)}$ be a DAG for $\widetilde{\mathbf{V}}^{(k+1)}\cup\widetilde{\mathbf{U}}^{(k+1)}$,
	then $\widetilde{\mathbf{X}}^{(k+1)} \subset An(\widetilde{Y}^{(k+1)})_{\widetilde{G}^{(k+1)}} \cup Ind(\widetilde{Y}^{(k+1)})_{\widetilde{G}^{(k+1)}}$.
	Moreover, $P(\widetilde{Y}^{(k+1)}|do(\widetilde{X}_{k+1}^{(k+1)}), do(\widetilde{X}_{k+2}^{(k+1)}), \ldots, do(\widetilde{X}_{n}^{(k+1)}))$
	is identifiable in $\widetilde{G}^{(k+1)}$.
	Now, note that it holds that
	\begin{equation}
		\widetilde{X}_{i}^{(k)} - E[\widetilde{X}_{i}^{(k)}|\widetilde{X}_{k}^{(k)}] = \widetilde{X}_{i}^{(k)} - E[\widetilde{X}_{i}^{(k)}] = \widetilde{X}_{i}^{(k)}, \nonumber
	\end{equation}
 	because $\widetilde{X}_{1}^{(k)}, \widetilde{X}_{2}^{(k)},\ldots, \widetilde{X}_{k-1}^{(k-1)}$ are independent of $\widetilde{X}_{k}^{(k)}$ for $i \le k-1$
 	from (\ref{eq_trans_lemma_one_do_minus_modify}) in lemma \ref{lemma_one_do_minus_modify_also_identifibility_preserve}.
	Namely, $\widetilde{X}_{i}^{(k)} (i \le k-1)$ are the same before and after operation of (\ref{eq_trans_lemma_one_do_minus_modify}).

	By the way, we already have shown the assumption \ref{cond-1} when $k+1$.
	Therefore, we can use $do$-calculus R3 to $P(\widetilde{Y}^{(k+1)}|do(\widetilde{\mathbf{X}}^{(k+1)}))$ in $\widetilde{G}^{(k+1)}$, 
	and we obtain
	\begin{eqnarray}
		&&P(\widetilde{Y}^{(k+1)}|do(\widetilde{\mathbf{X}}^{(k+1)}))  \nonumber \\
		&& \quad = P(\widetilde{Y}^{(k+1)}|do(\widetilde{X}_{1}^{(k+1)}), do(\widetilde{X}_{2}^{(k+1)}), \ldots,  do(\widetilde{X}_{n}^{(k+1)})) \nonumber \\
		&& \quad = P(\widetilde{Y}^{(k+1)}|do(\widetilde{X}_{k+1}^{(k+1)}), do(\widetilde{X}_{k+2}^{(k+1)}), \ldots, do(\widetilde{X}_{n}^{(k+1)})) \nonumber
	\end{eqnarray}
	Recall that the right hand side of above equations is identifiable in $\widetilde{G}^{(k+1)}$.
	As a result, $P(\widetilde{Y}^{(k+1)}|do(\widetilde{\mathbf{X}}^{(k+1)}))$ is identifiable in $\widetilde{G}^{(k+1)}$.
	
	Summarizing above results, the assumptions \ref{cond-1}-\ref{cond-4} holds when $k+1$.
	
	Next, we show that $\widetilde{G}$ is available for modeling the average causal effect $E[\widetilde{Y}|do(\widetilde{\mathbf{X}})]$.
	Note that $\widetilde{G} = \widetilde{G}^{(n)}$.
	From the inductive assumptions when $k=n$, it holds that for each $\widetilde{X}_i^{(i)}$,
	\begin{equation}
		\widetilde{X}_i^{(i)} \indep \widetilde{\mathbf{X}} \setminus \widetilde{X}_i^{(i)} 
		\quad and \quad \widetilde{X}_i^{(i)} \indep  \mathbf{Z} 
		\quad and \quad \widetilde{X}_i^{(i)} \in An(\widetilde{Y})_{\widetilde{G}} \cup Ind(\widetilde{Y})_{\widetilde{G}}. \nonumber
	\end{equation}
	That is, there exist no arrows emerging from $\widetilde{V}$ to each $\widetilde{X}_i^{(i)}$.
	Hence,
	\begin{equation}
		\overline{\widetilde{G}^{obs}}(\widetilde{\mathbf{X}})=\widetilde{G}^{obs} \nonumber
	\end{equation}
	As the above result, $\widetilde{G}$ is available for modeling the average causal effect $E[\widetilde{Y}|do(\widetilde{\mathbf{X}})]$.

	Finally, we show equations (\ref{eq_do_equals_codi})-(\ref{for_collolary_trans}).
	From (\ref{def_causal_effect_no_condi_eq}), 
	\begin{eqnarray}
		&&P(\widetilde{Y}^{(n)}=\widetilde{y}^{(n)}|do(\widetilde{X}_n^{(n)}=\widetilde{x}_{n}^{(n)}),
		do(\widetilde{X}_{n-1}^{(n-1)}=\widetilde{x}_{n-1}^{(n-1)}),
		\ldots, do(\widetilde{X}_1^{(1)}=\widetilde{x}_1^{(1)})) \nonumber\\
		&& \qquad= P(\widetilde{Y}^{(n)}=\widetilde{y}^{(n)}|\widetilde{X}_n^{(n)}=\widetilde{x}_n^{(n)},
		\widetilde{X}_{n-1}^{(n-1)}=\widetilde{x}_{n-1}^{(n-1)},
		\ldots, \widetilde{X}_1^{(1)}=\widetilde{x}_1^{(1)}) \nonumber\\
		&&  \qquad= P(\widetilde{Y}^{(n)}=\widetilde{y}^{(n)}|\widetilde{X}_n^{(n)}=\widetilde{x}_n^{(n)},
		\widetilde{X}_{n-1}^{(n-1)}=\widetilde{x}_{n-1}^{(n-1)},
		\ldots, \widetilde{X}_1^{(1)}=\widetilde{x}_1^{(1)}) \label{eq_theorem2_intervetnion_trans_pre}\\
		&&  \qquad= 
		\frac{P(\widetilde{Y}^{(n)}=\widetilde{y}^{(n)}, \widetilde{X}_n^{(n)}=\widetilde{x}_n^{(n)},
			\widetilde{X}_{n-1}^{(n-1)}=\widetilde{x}_{n-1}^{(n-1)},
			\ldots, \widetilde{X}_1^{(1)}=\widetilde{x}_1^{(1)})}
		{P(\widetilde{X}_n^{(n)}=\widetilde{x}_n^{(n)},
			\widetilde{X}_{n-1}^{(n-1)}=\widetilde{x}_{n-1}^{(n-1)},
			\ldots, \widetilde{X}_1^{(1)}=\widetilde{x}_1^{(1)}))}. \nonumber
	\end{eqnarray}
	Now that, from the definitions of $\widetilde{Y}^{(k)}$ and $\widetilde{\mathbf{X}}^{(k)}$ for $2 \le k \le n$,
	it holds that
	\begin{equation}
		P(\widetilde{Y}^{(k)}=\widetilde{y}^{(k)}) = P(\widetilde{Y}^{(k-1)}=\widetilde{y}^{(k-1)}|\widetilde{X}_{k-1}^{(k-1)}=\widetilde{x}_{k-1}^{(k-1)}) \nonumber
	\end{equation}
	\begin{equation}
		P(\widetilde{X}_{i}^{(k)}=\widetilde{x}_{i}^{(k)}) = P(\widetilde{X}_{i}^{(k-1)}=\widetilde{x}_{i}^{(k-1)}|\widetilde{X}_{k-1}^{(k-1)}=\widetilde{x}_{k-1}^{(k-1)}) \nonumber
	\end{equation}
	In addition, since we already have shown that
	\begin{equation}
		\widetilde{Y}^{(k)} \indep \{\widetilde{X}_1^{(1)}, \widetilde{X}_2^{(2)}, \dots,\widetilde{X}_{k-1}^{(k-1)}\} \quad and \quad
		\widetilde{X}_i^{(k)} \indep \{\widetilde{X}_1^{(1)}, \widetilde{X}_2^{(2)}, \dots,\widetilde{X}_{k-1}^{(k-1)}\} \ (\forall i \ge k), \label{eq_indep_n_and_less_n} 
	\end{equation}
	we obtain
	\begin{eqnarray}
		&& P(\widetilde{Y}^{(n)}=\widetilde{y}^{(n)}, \widetilde{X}_n^{(n)}=\widetilde{x}_n^{(n)},
		\widetilde{X}_{n-1}^{(n-1)}=\widetilde{x}_{n-1}^{(n-1)},
		\ldots, \widetilde{X}_1^{(1)}=\widetilde{x}_1^{(1)}) \nonumber \\
		&& = P(\widetilde{Y}^{(n)}=\widetilde{y}^{(n)}, \widetilde{X}_n^{(n)}=\widetilde{x}_n^{(n)}) \times P(\widetilde{X}_{n-1}^{(n-1)}=\widetilde{x}_{n-1}^{(n-1)}) \nonumber\\
		&& \qquad \times P(\widetilde{X}_{n-2}^{(n-2)}=\widetilde{x}_{n-2}^{(n-2)}) \times \cdots \times P(\widetilde{X}_1^{(1)}=\widetilde{x}_1^{(1)}) \nonumber\\
		&& = P(\widetilde{Y}^{(n-1)}=\widetilde{y}^{(n-1)}, \widetilde{X}_n^{(n-1)}=\widetilde{x}_{n}^{(n-1)} | \widetilde{X}_{n-1}^{(n-1)}=\widetilde{x}_{n-1}^{(n-1)}) \times 
		P(\widetilde{X}_{n-1}^{(n-1)}=\widetilde{x}_{n-1}^{(n-1)}) \nonumber \\ 
		&& \qquad \times P(\widetilde{X}_{n-2}^{(n-2)}=\widetilde{x}_{n-2}^{(n-2)}) \times \cdots \times P(\widetilde{X}_1^{(1)}=\widetilde{x}_1^{(1)}) \nonumber\\
		&& = P(\widetilde{Y}^{(n-1)}=\widetilde{y}^{(n-1)}, \widetilde{X}_n^{(n-1)}=\widetilde{x}_{n}^{(n-1)}, \widetilde{X}_{n-1}^{(n-1)}=\widetilde{x}_{n-1}^{(n-1)}) \times P(\widetilde{X}_{n-2}^{(n-2)}=\widetilde{x}_{n-2}^{(n-2)}) \nonumber\\
		&& \qquad \times P(\widetilde{X}_{n-3}^{(n-3)}=\widetilde{x}_{n-3}^{(n-3)}) \times \cdots \times P(\widetilde{X}_1^{(1)}=\widetilde{x}_1^{(1)}) \nonumber\\
		&& = P(\widetilde{Y}^{(n-2)}=\widetilde{y}^{(n-2)}, \widetilde{X}_n^{(n-2)}=\widetilde{x}_{n}^{(n-2)}, \nonumber\\
		&& \quad \qquad \widetilde{X}_{n-1}^{(n-2)}=\widetilde{x}_{n-1}^{(n-2)}|\widetilde{X}_{n-2}^{(n-2)}=\widetilde{x}_{n-2}^{(n-2)}) \times P(\widetilde{X}_{n-2}^{(n-2)}=\widetilde{x}_{n-2}^{(n-2)}) \nonumber\\
		&& \qquad \times P(\widetilde{X}_{n-3}^{(n-3)}=\widetilde{x}_{n-3}^{(n-3)}) \times \cdots \times P(\widetilde{X}_1^{(1)}=\widetilde{x}_1^{(1)}) \nonumber\\
		&& = P(\widetilde{Y}^{(n-2)}=\widetilde{y}^{(n-2)}, \widetilde{X}_n^{(n-2)}=\widetilde{x}_{n}^{(n-2)}, \widetilde{X}_{n-1}^{(n-2)}=\widetilde{x}_{n-1}^{(n-2)}, \nonumber\\
		&& \quad \qquad \widetilde{X}_{n-2}^{(n-2)}=\widetilde{x}_{n-2}^{(n-2)}) \times P(\widetilde{X}_{n-3}^{(n-3)}=\widetilde{x}_{n-3}^{(n-3)}) \nonumber\\
		&& \qquad \times P(\widetilde{X}_{n-4}^{(n-4)}=\widetilde{x}_{n-4}^{(n-4)}) \times \cdots \times P(\widetilde{X}_1^{(1)}=\widetilde{x}_1^{(1)}). \nonumber
	\end{eqnarray}
	By repeating this operation until k=n, we obtain
	\begin{eqnarray}
		&& P(\widetilde{Y}^{(n)}=\widetilde{y}^{(n)}, \widetilde{X}_n^{(n)}=\widetilde{x}_n^{(n)},
		\widetilde{X}_{n-1}^{(n-1)}=\widetilde{x}_{n-1}^{(n-1)},
		\ldots, \widetilde{X}_1^{(1)}=\widetilde{x}_1^{(1)}) \nonumber \nonumber\\
		&& =  P(\widetilde{Y}^{(1)}=\widetilde{y}^{(1)}, \widetilde{X}_n^{(1)}=\widetilde{x}_n^{(1)},
		\widetilde{X}_{n-1}^{(1)}=\widetilde{x}_{n-1}^{(1)},
		\ldots, \widetilde{X}_1^{(1)}=\widetilde{x}_1^{(1)}). \nonumber
	\end{eqnarray}
	In addition, from the definition of $\widetilde{Y}^{(1)}$ and $\widetilde{\mathbf{X}}^{(1)}$, we obtain
	\begin{equation}
		P(\widetilde{Y}^{(1)}=\widetilde{y}^{(1)}) = P(Y=y|\mathbf{Z}=\mathbf{z}) \nonumber
	\end{equation}
	\begin{equation}
		P(\widetilde{X}_i^{(1)}=\widetilde{x}_i^{(1)})=P(X_i=x_i|\mathbf{Z}=\mathbf{z}) \qquad (1 \le k \le n). \nonumber
	\end{equation}
	Thus, 
	\begin{eqnarray}
		&& P(\widetilde{Y}^{(1)}=\widetilde{y}^{(1)}, \widetilde{X}_n^{(1)}=\widetilde{x}_n^{(1)},
		\widetilde{X}_{n-1}^{(1)}=\widetilde{x}_{n-1}^{(1)},
		\ldots, \widetilde{X}_1^{(1)}=\widetilde{x}_1^{(1)}) \nonumber \\
		&& = P(Y=y, X_n=x_n , \ldots, X_1=x_1|\mathbf{Z}=\mathbf{z}). \nonumber
	\end{eqnarray}
	From the above results,
	\begin{eqnarray}
		&& P(\widetilde{Y}^{(n)}=\widetilde{y}^{(n)}, \widetilde{X}_n^{(n)}=\widetilde{x}_n^{(n)},
		\widetilde{X}_{n-1}^{(n-1)}=\widetilde{x}_{n-1}^{(n-1)},
		\ldots, \widetilde{X}_1^{(1)}=\widetilde{x}_1^{(1)}) \nonumber \\
		&& \quad =P(Y=y, X_n=x_n, X_{n-1}=x_{n-1} \ldots, X_1=x_1|\mathbf{Z}=\mathbf{z}). \label{eq_intevention_bunshi}
	\end{eqnarray}
	Similarly, we can also obtain
	\begin{eqnarray}
		&& P(\widetilde{X}_n^{(n)}=\widetilde{x}_n^{(n)}, \widetilde{X}_{n-1}^{(n-1)}=\widetilde{x}_{n-1}^{(n-1)},
		\ldots, \widetilde{X}_1^{(1)}=\widetilde{x}_1^{(1)}) \nonumber\\
		&& \quad =P(X_n=x_n , X_{n-1}=x_{n-1}, \ldots, X_1=x_1|\mathbf{Z}=\mathbf{z}).  \label{eq_intevention_bunbo}
	\end{eqnarray}
	By inserting (\ref{eq_intevention_bunshi}) and (\ref{eq_intevention_bunbo}) to (\ref{eq_theorem2_intervetnion_trans_pre}),
	\begin{eqnarray}
		&& \frac{P(\widetilde{Y}^{(n)}=\widetilde{y}^{(n)}, \widetilde{X}_n^{(n)}=\widetilde{x}_n^{(n)},
			\widetilde{X}_{n-1}^{(n-1)}=\widetilde{x}_{n-1}^{(n-1)},
			\ldots, \widetilde{X}_1^{(1)}=\widetilde{x}_1^{(1)})}
		{P(\widetilde{X}_n^{(n)}=\widetilde{x}_n^{(n)},
			\widetilde{X}_{n-1}^{(n-1)}=\widetilde{x}_{n-1}^{(n-1)},
			\ldots, \widetilde{X}_1^{(1)}=\widetilde{x}_1^{(1)}))} \nonumber \\
		&& \quad = \frac{P(Y=y, X_n=x_n, X_{n-1}=x_{n-1} \ldots, X_1=x_1|\mathbf{Z}=\mathbf{z})}
		{P(X_n=x_n , X_{n-1}=x_{n-1}, \ldots, X_1=x_1|\mathbf{Z}=\mathbf{z})} \nonumber\\
		&& \quad = P(Y=y | X_n=x_n, X_{n-1}=x_{n-1} \ldots, X_1=x_1, \mathbf{Z}=\mathbf{z}). \label{eq_tilde_is_same_as_condi}
	\end{eqnarray}
	On the other hand, by lemma \ref{lemma_intervention_eq},
	\begin{eqnarray}
		\lefteqn{P(Y=y|do(X_1=x_1),\dots,do(X_n=x_n), \mathbf{Z}=\mathbf{z})} \quad \nonumber\\
		&=& P(Y=y | X_n=x_n, X_{n-1}=x_{n-1} \ldots, X_1=x_1, \mathbf{Z}=\mathbf{z}). \label{eq_do_is_same_as_condi}
	\end{eqnarray}
	Thus, from (\ref{eq_tilde_is_same_as_condi}) and  (\ref{eq_do_is_same_as_condi}), (\ref{eq_tilde_is_same_original}) holds.
	
	Finally, by (\ref{eq_indep_n_and_less_n}) when $k=n$,
	\begin{eqnarray}
		&& P(\widetilde{Y}^{(n)}=\widetilde{y}^{(n)}|\widetilde{X}_{1}^{(1)}=\widetilde{x}_{1}^{(1)},
		\widetilde{X}_{2}^{(2)}=\widetilde{x}_{2}^{(2)},\dots,\widetilde{X}_{n}^{(n)}=\widetilde{x}_{n}^{(n)}) \nonumber\\
		&& \quad = \frac{P(\widetilde{Y}^{(n)}=\widetilde{y}^{(n)}, \widetilde{X}_{1}^{(1)}=\widetilde{x}_{1}^{(1)},
			\widetilde{X}_{2}^{(2)}=\widetilde{x}_{2}^{(2)},\dots,\widetilde{X}_{n}^{(n)}=\widetilde{x}_{n}^{(n)})}
		{P(\widetilde{X}_{1}^{(1)}=\widetilde{x}_{1}^{(1)},\widetilde{X}_{2}^{(2)}=\widetilde{x}_{2}^{(2)},
			\dots,\widetilde{X}_{n}^{(n)}=\widetilde{x}_{n}^{(n)})} \nonumber \\
		&& \quad = \frac{P(\widetilde{Y}^{(n)}=\widetilde{y}^{(n)}, \widetilde{X}_{n}^{(n)}=\widetilde{x}_{n}^{(n)})
			\cdot P(\widetilde{X}_{1}^{(1)}=\widetilde{x}_{1}^{(1)}, %\widetilde{X}_{2}^{(2)}=\widetilde{x}_{2}^{(2)},
			\dots,\widetilde{X}_{n-1}^{(n-1)}=\widetilde{x}_{n-1}^{(n-1)})}
		{P(\widetilde{X}_{n}^{(n)}=\widetilde{x}_{n}^{(n)}) \cdot P(\widetilde{X}_{1}^{(1)}=\widetilde{x}_{1}^{(1)},
			%\widetilde{X}_{2}^{(2)}=\widetilde{x}_{2}^{(2)},
			\dots,\widetilde{X}_{n-1}^{(n-1)}=\widetilde{x}_{n-1}^{(n-1)})}  \nonumber\\
		&& \quad = \frac{P(\widetilde{Y}^{(n)}=\widetilde{y}^{(n)}, \widetilde{X}_{n}^{(n)}=\widetilde{x}_{n}^{(n)})}
		{P(\widetilde{X}_{n}^{(n)}=\widetilde{x}_{n}^{(n)})} \nonumber\\
		&& \quad = P(\widetilde{Y}^{(n)}=\widetilde{y}^{(n)} | \widetilde{X}_{n}^{(n)}=\widetilde{x}_{n}^{(n)}). \nonumber
	\end{eqnarray}
	Thus, (\ref{for_collolary_trans}) holds.
	
	This completes the proof of the theorem.
\end{proof}

\begin{proof}[proof of Corollary \ref{corollary_how_to_calc_ave_intervention_effect}]
	In the proof of Theorem \ref{thorem_multi_vals}, we obtain the following results.
	\begin{equation}
		\widetilde{Y} \indep  \widetilde{\mathbf{X}} \label{eq_from_theo_2_top}
	\end{equation}
	\begin{equation}
		\widetilde{X}_i^{(i)} \indep \widetilde{X}_j^{(j)}  \quad (1 \le i < j \le n) 
	\end{equation}
	\begin{equation}	
		\widetilde{Y}, \widetilde{\mathbf{X}} \indep \mathbf{Z} 
	\end{equation}
	\begin{eqnarray}
		\lefteqn{P'(Y=y|X_1=x_1, X_2=x_2, \ldots, X_n=x_n, \mathbf{Z}=\mathbf{z})} \quad \nonumber \\
		&=&  P(Y=y|X_1=x_1, X_2=x_2, \ldots, X_n=x_n, \mathbf{Z}=\mathbf{z}) 
	\end{eqnarray}
	\begin{eqnarray}
		\lefteqn{P(\widetilde{X}_{1}^{(1)}=\widetilde{x}_{1}^{(1)},
			\widetilde{X}_{2}^{(2)}=\widetilde{x}_{2}^{(2)},\dots,\widetilde{X}_{n}^{(n)}=\widetilde{x}_{n}^{(n)},\mathbf{Z}=\mathbf{z})} \nonumber\\
		&=& P(X_1=x_1, X_2=x_2,\ldots,X_n=x_n|\mathbf{Z}=\mathbf{z}). \label{eq_from_theo_2_bottom}
	\end{eqnarray}
	From this, we obtain
	\begin{eqnarray}
		\lefteqn{P''(Y=y|\widetilde{X}_{1}^{(1)}=\widetilde{x}_{1}^{(1)},
			\widetilde{X}_{2}^{(2)}=\widetilde{x}_{2}^{(2)},}\nonumber \\
		&& \quad \dots,\widetilde{X}_{n}^{(n)}=\widetilde{x}_{n}^{(n)}, \mathbf{Z}=\mathbf{z}) \nonumber \\
		&& = P''(Y=y | \widetilde{X}_{1}^{(1)}=\widetilde{x}_{1}^{(1)}, \widetilde{X}_{2}^{(2)}=\widetilde{x}_{2}^{(2)},
		\dots,\widetilde{X}_{n}^{(n)}=\widetilde{x}_{n}^{(n)})  \nonumber \\
		&& = P(Y=y | \widetilde{X}_{1}^{(1)}=\widetilde{x}_{1}^{(1)}, \widetilde{X}_{2}^{(2)}=\widetilde{x}_{2}^{(2)},
		\dots,\widetilde{X}_{n}^{(n)}=\widetilde{x}_{n}^{(n)})  \nonumber \\
		&& = P(Y=y |X_1=x_1, X_2=x_2,\ldots,X_n=x_n, \mathbf{Z}=\mathbf{z}) \nonumber\\
		&& = P'(Y=y|X_1=x_1, X_2=x_2, \ldots, X_n=x_n, \mathbf{Z}=\mathbf{z}). \nonumber
	\end{eqnarray}
	Thus, (\ref{eq_condi_ave_intervent_to_tidle_trans}) holds.
	
	Now that, it holds that for each $\widetilde{X}_{i}^{(i)}(1\le i \le n)$, $\mathbf{Z}$
	and all continuous functions $f(x)$ and $g(\mathbf{z})$, 
	\begin{eqnarray}
		\lefteqn{E_{P''}[f(\widetilde{X}_{i}^{(i)})|\widetilde{X}_{1}^{(1)}=\widetilde{x}_{1}^{(1)},\widetilde{X}_{2}^{(2)}=\widetilde{x}_{2}^{(2)},} \nonumber \\
		&& \qquad \ldots,\widetilde{X}_{n}^{(n)}=\widetilde{x}_{n}^{(n)}, \mathbf{Z}=\mathbf{z}] \nonumber \\
		&=& f(\widetilde{x}_{i}^{(i)}) \label{eq_condi_only_remain_x_tilde}
	\end{eqnarray}
	\begin{eqnarray}
		\lefteqn{E_{P''}[g(\mathbf{Z})|\widetilde{X}_{1}^{(1)}=\widetilde{x}_{1}^{(1)},\widetilde{X}_{2}^{(2)}=\widetilde{x}_{2}^{(2)},} \nonumber \\
		&& \qquad \ldots,\widetilde{X}_{n}^{(n)}=\widetilde{x}_{n}^{(n)}, \mathbf{Z}=\mathbf{z}] \nonumber \\
		&=& g(\mathbf{z}).  \label{eq_condi_only_remain_z}
	\end{eqnarray}
	In fact, for $P''$, it holds that for $1 \le i < j \le n$,
	\begin{equation}
		\widetilde{X}_{i}^{(i)} \indep 	\widetilde{X}_{j}^{(j)}. \nonumber
	\end{equation}
	In addition, for $P''$, it hold that for $1 \le i \le n$,
	\begin{equation}
		\widetilde{X}_{i}^{(i)} \indep \mathbf{Z}. \nonumber
	\end{equation}
	Therefore,
	\begin{eqnarray}
		\lefteqn{E_{P''}[f(\widetilde{X}_{i}^{(i)})|\widetilde{X}_{1}^{(1)}=\widetilde{x}_{1}^{(1)},\widetilde{X}_{2}^{(2)}=\widetilde{x}_{2}^{(2)},} \quad \nonumber \\
		&& \qquad \ldots,\widetilde{X}_{n}^{(n)}=\widetilde{x}_{n}^{(n)}, \mathbf{Z}=\mathbf{z}] \nonumber\\
		&=& E_{P''}[f(\widetilde{X}_{i}^{(i)})|\widetilde{X}_{i}^{(i)}=\widetilde{x}_{i}^{(i)}] \nonumber\\
		&=& f(\widetilde{x}_{i}^{(i)}), \nonumber
	\end{eqnarray}
	and 
	\begin{eqnarray}
		\lefteqn{E_{P''}[g(\mathbf{Z})|\widetilde{X}_{1}^{(1)}=\widetilde{x}_{1}^{(1)},\widetilde{X}_{2}^{(2)}=\widetilde{x}_{2}^{(2)},} \quad \nonumber \\
		&& \qquad \ldots,\widetilde{X}_{n}^{(n)}=\widetilde{x}_{n}^{(n)}, \mathbf{Z}=\mathbf{z}] \nonumber\\
		&=& E_{P''}[g(\mathbf{Z})|\mathbf{Z}=\mathbf{z}] \nonumber\\
		&=& g(\mathbf{z}). \nonumber
	\end{eqnarray}
	Recall that from the definition of $\widetilde{Y}^{(1)}, \widetilde{Y}^{(2)}, \ldots , \widetilde{Y}^{(n)}$,
	\begin{eqnarray}
		&& Y= E_{P}[Y|\mathbf{Z}=\mathbf{z}] + E_{P}[\widetilde{Y}^{(1)}|\widetilde{X}_{1}^{(1)}=\widetilde{x}_{1}^{(1)}]  \nonumber\\
		&& \qquad + \cdots  + E_{P}[\widetilde{Y}^{(n-1)}|\widetilde{X}_{n-1}^{(n-1)}=\widetilde{x}_{n-1}^{(n-1)}]  \nonumber \\
		&& \qquad + \widetilde{Y}^{(n)}. \nonumber
	\end{eqnarray}
	From the above results, we obtain
	\begin{eqnarray}
		\lefteqn{E_{P''}[Y|\widetilde{X}_{1}^{(1)}=\widetilde{x}_{1}^{(1)}, \widetilde{X}_{2}^{(2)}=\widetilde{x}_{2}^{(2)},} \nonumber \\
		&& 	\ldots,\widetilde{X}_{n}^{(n)}=\widetilde{x}_{n}^{(n)}]  \nonumber \\
		&=&  E_{P''}[E_{P}[Y|\mathbf{Z}=\mathbf{z}] + E_{P}[\widetilde{Y}^{(1)}|\widetilde{X}_{1}^{(1)}=\widetilde{x}_{1}^{(1)}] \nonumber \\
		&& \quad  + \cdots  + E_{P}[\widetilde{Y}^{(n-1)}|\widetilde{X}_{n-1}^{(n-1)}=\widetilde{x}_{n-1}^{(n-1)}] \nonumber \\
		&& \quad  + \widetilde{Y}^{(n)} | \widetilde{X}_{1}^{(1)}=\widetilde{x}_{1}^{(1)}, \widetilde{X}_{2}^{(2)}=\widetilde{x}_{2}^{(2)}, \nonumber \\
		&& \qquad \qquad \qquad \ldots,\widetilde{X}_{n}^{(n)}=\widetilde{x}_{n}^{(n)}, \mathbf{Z}=\mathbf{z}] \nonumber \\
		&=& E_{P}[Y|\mathbf{Z}=\mathbf{z}] + E_{P}[\widetilde{Y}^{(1)}|\widetilde{X}_{1}^{(1)}=\widetilde{x}_{1}^{(1)}] \nonumber \\
		&& \quad + \cdots + E_{P}[\widetilde{Y}^{(n-1)}|\widetilde{X}_{n-1}^{(n-1)}=\widetilde{x}_{n-1}^{(n-1)}] \nonumber \\
		&& \quad + E_{P''}[\widetilde{Y}^{(n)}|\widetilde{X}_{1}^{(1)}=\widetilde{x}_{1}^{(1)}, 
		\widetilde{X}_{2}^{(2)}=\widetilde{x}_{2}^{(2)}, \nonumber \\
		&& \qquad \qquad \qquad \ldots, \widetilde{X}_{n}^{(n)}=\widetilde{x}_{n}^{(n)}, \mathbf{Z}=\mathbf{z}] \label{trans_only_remain}\\
		&=& E_{P}[Y|\mathbf{Z}=\mathbf{z}] + E_{P}[\widetilde{Y}^{(1)}|\widetilde{X}_{1}^{(1)}=\widetilde{x}_{1}^{(1)}]  \nonumber \\
		&& \quad + \cdots + E_{P}[\widetilde{Y}^{(n-1)}|\widetilde{X}_{n-1}^{(n-1)}=\widetilde{x}_{n-1}^{(n-1)}] \nonumber \\
		&& \quad + E_{P}[\widetilde{Y}^{(n)}|\widetilde{X}_{n}^{(n)}=\widetilde{x}_{n}^{(n)}].  \label{trans_from_theorem_for_muli_intv}
	\end{eqnarray}
	Now, we obtain (\ref{trans_only_remain}) from (\ref{eq_condi_only_remain_x_tilde}) and (\ref{eq_condi_only_remain_z}),
	and we obtain (\ref{trans_from_theorem_for_muli_intv}) from (\ref{for_collolary_trans}) in Theorem \ref{thorem_multi_vals}.
	Thus, (\ref{eq_for_estimate_average _effect}) holds.
	
	This completes the proof of the corollary.
\end{proof}

\end{document}